\documentclass[a4paper,UKenglish,cleveref,autoref,thm-restate]{lipics-v2021}
\usepackage{times}  

\usepackage{hyperref}
\usepackage{newfloat}
\usepackage{listings}
\title{Diversity of Extensions in Abstract Argumentation}
\titlerunning{Diversity of Extensions in Abstract Argumentation}

\author{Johannes K. Fichte}{Department of Computer and Information Science (IDA), Link\"oping University, Sweden}{johannes.fichte@liu.se}{https://orcid.org/0000-0002-8681-7470}{}

\author{Markus Hecher}{University of Potsdam, Germany \& University of Artois, CNRS, UMR8188 (CRIL), France }{hecher@cril.fr}{https://orcid.org/0000-0003-0131-6771}{}

\author{Yasir Mahmood}{Data Science Group, Heinz Nixdorf Institute, Paderborn University, Germany}{yasir.mahmood@uni-paderborn.de}{https://orcid.org/0000-0002-5651-5391}{}

\author{Zhengjun Wang}{Data Science Group, Heinz Nixdorf Institute, Paderborn University, Germany}{zwang@mail.uni-paderborn.de}{}{}

\authorrunning{J.K.~Fichte, M.~Hecher, Y.~Mahmood, and Z.~Wang}
\Copyright{Yasir Mahmood, Markus Hecher, Johannes K. Fichte and Zhengjun Wang}



\nolinenumbers 

\hideLIPIcs  

\usepackage{soul}
\usepackage{amsthm}
\usepackage{booktabs}
\usepackage{algorithm}
\usepackage{algorithmic}
\usepackage{tikz}
\usepackage{macros}
\usepackage{multicol}
\usepackage{bibentry}

\usepackage{pgf,tikz}
\usetikzlibrary{fit,positioning,arrows,shapes, decorations,automata,%
	petri,%
	topaths,calc}
\tikzstyle{dashedarrow} = [-stealth',dashed]
\tikzstyle{tdnode} = [draw,rounded corners,top color=vertexTopColor,bottom color=vertexBottomColor,minimum size=1.5em]
\tikzstyle{stdnode} = [tdnode, font=\scriptsize]
\tikzstyle{stdnodecompact} = [stdnode, inner sep = 1.5pt, outer sep = 0.1pt]
\tikzstyle{stdnodetable} = [stdnode, inner sep = 1.5pt, outer sep = 0]
\tikzstyle{stdnodenum} = [minimum size=1.5em, font=\scriptsize]
\tikzstyle{tdedge} = [-,draw,thick]
\tikzstyle{tdlabel} = [draw=none, rectangle, fill=none, inner sep=0pt, font=\scriptsize]

\newcommand{\futuresketch}[1]{}
\newcommand{\future}[1]{}

\usepackage{xcolor}
\colorlet{vertexTopColor}{white}

\colorlet{vertexBottomColor}{black!10}

\allowdisplaybreaks

\newcommand{\lefttriangle}{\hfill$\triangleleft$}
\AtEndEnvironment{claimproof}{\filledsquare}
\AtEndEnvironment{example}{\lefttriangle}

\usepackage{enumerate,enumitem}

\makeatletter
\providecommand*{\cupdot}{%
	\mathbin{%
		\mathpalette\@cupdot{}%
	}%
}
\newcommand*{\@cupdot}[2]{%
	\ooalign{%
		$\m@th#1\cup$\cr
		\hidewidth$\m@th#1\cdot$\hidewidth
	}%
}
\makeatother

\newcommand{\symdif}{\triangle}
\newcommand{\yasir}[1]{\textcolor{blue}{#1}}
\newcommand{\alert}[1]{\textcolor{red}{#1}}

\newcommand{\kdiversearg}{\mathsf{k\text-Div\text-Arg}}
\newcommand{\kdiverseext}{\mathsf{k\text-Div\text-Ext}}
\newcommand{\kdiversecover}{\mathsf{k\text-Div\text-Cov}}
\newcommand{\maxkdiversecover}{\mathsf{max\text-k\text-Div\text-Cov}}
\newcommand{\minkdiversecover}{\mathsf{min\text-k\text-Div\text-Cov}}
\newcommand{\maxkdiversearg}{\mathsf{max\text-k\text-Div\text-Arg}}

\newcommand{\maxkdiverseext}{\mathsf{max\text-k\text-Div\text-Ext}}
\newcommand{\minkdiverseext}{\mathsf{min\text-k\text-Div\text-Ext}}

\newcommand{\belowkdiversearg}{\,^{\leq}\mathsf{k\text-Div\text-Arg}}
\newcommand{\abovekdiversearg}{\,^{\geq}\mathsf{k\text-Div\text-Arg}}
\newcommand{\belowkdiverseext}{\,^{\leq}\mathsf{k\text-Div\text-Ext}}
\newcommand{\abovekdiverseext}{\,^{\geq}\mathsf{k\text-Div\text-Ext}}

\usepackage{microtype}

\usepackage{thmtools}
\usepackage{thm-restate}
\usepackage{cleveref}
\usepackage{apptools}
\usepackage{xpatch}
\makeatletter
\xpatchcmd{\thmt@restatable}
{\csname #2\@xa\endcsname\ifx\@nx#1\@nx\else[{#1}]\fi}
{\IfAppendix{\csname #2\@xa\endcsname}{\csname #2\@xa\endcsname\ifx\@nx#1\@nx\else[{#1}]\fi}}
{}{} 
\makeatother

\usepackage{boxedminipage}
\newcommand{\pbDef}[3]{%
	\noindent
	\begin{center}
		\begin{boxedminipage}{0.98 \columnwidth}
			#1\\[1pt]
			\begin{tabular}{l p{0.70 \columnwidth}}
				Input: & #2\\
				Question: & #3
			\end{tabular}
		\end{boxedminipage}
	\end{center}
}

\newcommand{\pbDefT}[3]{%
	\noindent
	\begin{center}
		\begin{boxedminipage}{0.98 \columnwidth}
			#1\\[5pt]
			\begin{tabular}{l p{0.70 \columnwidth}}
				Input: & #2\\
				Task: & #3
			\end{tabular}
		\end{boxedminipage}
	\end{center}
}
\usepackage{subcaption}

\usepackage{thmtools}
\usepackage{thm-restate}
\usepackage{cleveref}
\usepackage{apptools}
\usepackage{xpatch}
\makeatletter
\xpatchcmd{\thmt@restatable}
{\csname #2\@xa\endcsname\ifx\@nx#1\@nx\else[{#1}]\fi}
{\IfAppendix{\csname #2\@xa\endcsname}{\csname #2\@xa\endcsname\ifx\@nx#1\@nx\else[{#1}]\fi}}
{}{} 
\makeatother

\begin{document}

\maketitle

\begin{abstract}
	Argumentation is an important topic of AI for modeling and reasoning
	about arguments. In abstract argumentation, we consider directed
	graphs, so-called \emph{argumentation frameworks (AF)}, that express
	conflicts between arguments.
	The semantics is defined by the notion of \emph{extensions}, which
	are sets of arguments that satisfy particular relationship
	conditions in the AF. Usually, 
	%
	standard reasoning in argumentation do not reveal \emph{how
		far apart} extensions are.
	%
	%
	
	We introduce a quantitative notion of \emph{diversity of extensions}
	based on the symmetric-difference and provide a systematic
	complexity classification.
	%
	%
	%
	Intuitively, diversity captures whether extensions of a
	framework (accepted viewpoints) differ only marginally
	or represent fundamentally incompatible sets of arguments.  We
	study whether an AF admits k-diverse 
	extensions, admits k-diverse 
	extensions covering
	specific arguments, 
	and to compute the largest~$k$
	for which an AF admits $k$-diverse extensions.
	%
	%
	We outline a prototype and provide an evaluation for
	computing 
	diversity levels.
	
\end{abstract}

\section{Introduction}
Abstract
argumentation~\cite{Dung95a,Rahwan07a,BaroniCaminadaGiacomin11} is a
well-known concept in AI for modeling, relating, and grouping
arguments~\shortversion{\cite{AmgoudPrade09a,RagoCocarascuToni18a}}.
To express opposing positions, arguments are related to each other in
form of a directed graph, so-called \emph{argumentation framework
	(AF)}.
Sets of arguments (\emph{extensions}), which satisfy certain
conditions regarding the relationship among them (semantics), build
the solutions to an AF.
Usually, an AF admits multiple extensions, in other words, mutually
incompatible but internally coherent sets of jointly acceptable
arguments, which can be understood as alternative ``view-points'' on
the same argumentative situation.
%
%

Research in abstract argumentation has 
introduced concepts to facilitate insights into diverging extensions,
for example,
explaining and selecting a set of most representative
view-points~\cite{BernreiterMalyNardi24},
%
%
%
deciding between multiple extensions~\cite{DauphinCramerVan-der-Torre18},
%
exploring large solution spaces~\cite{DachseltGagglKrotzsch22},
%
%
%
deciding the significance of individual
arguments~\cite{MahmoodEtAl25}, 
%
conditional probability of one or multiple arguments occurring in an extension~\cite{FichteH0M23},
counting extensions
\cite{FichteHecherMeier24}, and 
incorporating preordering over arguments (extension-ranking
semantics)~\cite{SkibaRienstraThimm21}
%
%
%

However, existing concepts do not reveal \emph{how far apart} and
\emph{diverse} entire extensions or \emph{how diverse AFs} are.
This is particularly interesting in environments where we want to
negotiate positions. Different view-points may reasonably endorse
different extensions. If we consider only one arbitrarily chosen
extension, a representative extension, or parts of extensions, we
easily hide genuine trade-offs and impede agreement.
%
%
%
%
%
%
%
%
%
%
%
%
To this end, we study diverse extensions and 
%
%
%
%
%
introduce a quantitative notion of extension diversity based on the
symmetric-difference.
%
The \emph{symmetric-difference}~$A \triangle B$ of two sets~$A$ and
$B$ takes the elements that are in exactly one of them, but not in
both, i.e., $A\triangle B \eqdef (A\cup B)\setminus (A\cap B)$.
Example~\ref{ex:running} illustrates a situation where we are
interested in understanding the diversity of extensions.

\begin{figure}
	\centering
	\includegraphics{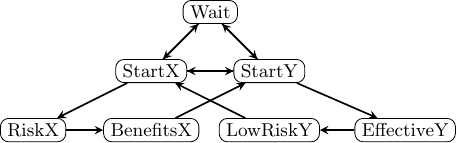}
	\caption{An Example AF modeling the medical decision making.}
	\label{ex:intro}
\end{figure}

\begin{example}\label{ex:running}
	Take an AF
	$F = (A,R)$ modeling a medical decision under time pressure, as depicted in Figure~\ref{ex:intro}.
	The arguments of $F$ have the following intuitive meaning:
	{\begin{itemize} \setlength{\itemsep}{0pt}%
			\setlength{\parskip}{0pt}%
			\item \textbf{sx} ({StartX}): Start Treatment X immediately.
			\item \textbf{sy} (StartY): Start Treatment Y immediately.
			\item \textbf{rx} ({RiskX}): Given no immediate decision, X's long-term risks must be carefully evaluated first.
			\item \textbf{ey} ({EffectiveY}): Given no immediate decision, Y's effectiveness must be further tested.
			\item \textbf{bx} {(BenefitsX)}: X has strong short-term benefits.
			\item \textbf{lry} ({LowRiskY}): Y has fewer side effects.
			\item \textbf{w} ({Wait}): collect more data before deciding.
	\end{itemize}}
	Under stable semantics, which intuitively expresses a
	self-consistent position that defeats all alternatives, $F$ admits exactly three extensions. 
	Precisely, 
	$E_X = \{sx,bx,ey\}, E_Y = \{sy,lry,rx\}$, and $E_W=\{w,rx,ey\}$,
	each highlighting either a type of treatment, or the argument to
	wait for more~data.
	
	Here, three most important arguments regarding the final decision are $sx, sy$ and $w$ which either select a type of treatment or argue to wait, respectively.
	Each of these arguments occur in some
	extension (credulously accepted).  Hence, we cannot tell much about
	either of the three options.
	But looking from a global perspective on all the available options (stable extensions),
	we observe that the extensions for either treatment do not share any
	argument, whereas they do share at least one argument with the
	extension that forces us to wait ($E_W$).
	Precisely, 
	$E_X \triangle E_Y = \{sx,bx,ey, sy,lry,rx\}$ meaning we consider options ``StartX'' and ``StartY'' to be $6$-diverse.
	%
	%
	%
	%
	However, we have $ E_X \triangle E_W = \{sx,bx,rx,w\}$, and
	$ E_Y \triangle E_W = \{sy,lry,ey,w\}$. 
	Consequently, we observe that
	the options ``StartX'' (``StartY'') and ``Wait'' are only $4$-diverse.
	Intuitively, this means that the justification for either type of treatment uses very ``different'' line of reasoning, whereas both treatment options share at least some reasoning with the third option to wait before deciding anything.
	Alternatively, the reasoning for ``StartX'' is somewhat \emph{similar} to the reasoning for ``Wait'' compared to ``StartY''.
\end{example}

\noindent\textbf{Contributions.} Our main contributions are as follows.
{\begin{enumerate}\setlength{\itemsep}{0pt}%
		\setlength{\parskip}{0pt}%
		\item We introduce the notion of symmetric-difference-based diversity
		to abstract argumentation and study its properties. We show how
		diversity differs from existing notions and that it provides a more fine-grained notion where concepts based
		on credulous and skeptical reasoning fail.
		\item We study the computational complexity of central problems
		(above, below, exactly $k$-diverse, max diversity) and provide a
		thorough 
		picture of the complexity landscape, which are listed in
		Table~\ref{tab:results}.
		\item We provide a prototypical implementation and an initial
		experimental evaluation, where we compute the range of diversity via
		logic programming (ASP).
\end{enumerate}}

\newcommand{\tref}[1]{\ensuremath{^{\footnotesize{\color{blue}\text{ Th.}\ref{#1}}}}}
\newcommand{\trefs}[2]{\ensuremath{^{\footnotesize{\color{blue}\text{ Th.}\ref{#1}}/\ref{#2}}}}
\newcommand{\rref}[1]{\ensuremath{^{\footnotesize{\color{blue}\text{ R.}\ref{#1}}}}}

\begin{table}
	\centering
	\begin{tabular}{lccc}%
			\toprule%
			Problems/$\sem$            & $\sigma_1 $                                                                & $\sigma_2$                                       & $\sigma_3 $                        \\
			\midrule
			$\belowkdiverseext_\sem$  & $\in\NP /$ Trivial\rref{rem:conf-atmost-div-arg}                           & $\NP$-c\tref{thm:div-ext}                        & $\SigmaP$-c\tref{thm:div-ext-semi} \\
			$\abovekdiverseext_\sem$ & $\NP$-c\tref{thm:conf-atleast-div-ext}                                       & $\NP$-c\tref{thm:div-ext}                        & $\SigmaP$-c\tref{thm:div-ext-semi} \\
			$\belowkdiversearg_\sem$  & $\NP$-c\tref{thm:nai-atmost-div-arg}/$\in\Ptime\tref{rem:conf-atmost-div-arg}$ & $\NP$-c\tref{thm:div-arg}                        & $\SigmaP$-c\tref{thm:div-arg-semi} \\
			$\abovekdiversearg_\sem$ & $\NP$-c\tref{thm:conf-atleast-div-arg}                                       & $\NP$-c\tref{thm:div-arg}                        & $\SigmaP$-c\tref{thm:div-arg-semi} \\
			Max                       & $\DP$-c\tref{thm:conf-max-k-div}$^/$\tref{thm:conf-max-k-div-arg}                                                                          & $\DP$-c\tref{thm:max-k-div}                       & $\DPtwo$-c\tref{thm:max-k-div}\\ 
			\bottomrule
	\end{tabular}
	\caption{%
		Overview of our complexity results for the semantics~$\sigma_1\in \{\cnf, \nai\}$,
		$\sigma_2\in \{\adm, \stab, \comp\}$, and 
		$\sigma_3\in \{\semi,$ $\stag\}$.
		The problem ``Arg'' asks if two arguments are $k$-diverse and ``Ext'' asks if an AF has two $k$-diverse $\sem$ extensions.
		The results for $\belowkdiversearg_\sem$ also apply to $\kdiversearg_\sem$ for each $\sem$.
	}
	\label{tab:results}%
\end{table}

Proofs to our statements can be found in the technical appendix attached at the end of this pdf . The implementation code is available online in our GitHub Repository~\footnote{\url{https://github.com/Yahahasir/argu_diversity/}}.

\paragraph{Related Work.}

\cite{CyrasRagoAlbini21} survey argumentation-based explanation under
the perspective of XAI.
\cite{LiaoTorre20} introduce and study explanation semantics, which
associates with each accepted argument a set of such explanation
arguments.
\cite{NiskanenJarvisalo20a} and \cite{SaribaturWallnerWoltran20}
study the opposite case where arguments are rejected and analyze
reasons for the rejection.
Although counting complexity for argumentation has been studied
\cite{FichteHecherMeier24}, the results do not apply here, as our focus is on the distances between elements in the extension space.
There are various practical solvers for decision and reasoning
tasks~\cite{EglyGagglWoltran08,NiskanenJarvisalo20,ThimmCeruttiVallati21,Alviano18,Alviano21}
that compete in the biannual ICCMA competition~\cite{ThimmEtAl24}.

\cite{MahmoodEtAl25} propose facets to abstract argumentation, which enable quantifying significance of decisions and study the
computational complexity. Facets do not capture information on extensions containing them, whereas diversity measures the (dis)similarity of accepted argument sets in an AF.
\cite{bohl2023representative} consider diversity and representative
sets in answer-set programming.
\cite{UlbrichtWallner21} consider strong explanation, which ensure
that a target set of arguments is acceptable in each sub-framework
containing the explaining set.
\cite{BesnardDoutreDuchatelle22} investigate concepts of contrastive
and non-contrastive questions with respect to explanations.
%
%
Finding similar and diverse solutions has been considered for
constraint programming~\cite{HebrardHnichOSullivan05} and answer-set
programming~\cite{EiterErdemErdogan13}.
\cite{Rodrigues19} compare different alternative representations of
extensions for systematic enumeration.
\cite{Coste-MarquisKoniecznyMailly15} employed Hamming distance to
define enforcement operators as previously also done
in the context of minimal
change~\cite{Baumann12}.
%
%
%
%
\cite{DachseltGagglKrotzsch22} developed a tool to navigate
argumentation frameworks using ASP-facets and more fine-grained
reasoning modes have also been studied in
ASP~\cite{FichteHecherNadeem22}.  Note that ASP-navigation is based on
forbidding or enforcing atoms in programs via integrity constraints.
%
%

\section{Preliminaries}

We assume familiarity with computational
complexity~\cite{DBLP:books/daglib/0092426}, graph
theory~\cite{DBLP:books/sp/BondyM08}, and Boolean
logic~\cite{DBLP:series/faia/336}.

\paragraph{Complexity Classes.}
We use standard notation for basic complexity classes and for example write $\Ptime$ ($\NP$) for the class of decision problems solvable in (non-deterministic) polynomial time. Additionally, we let $\co \NP$ be the class of decision problems whose complement is in $\NP$, and let $\DP$ be the class of decision problems representable as the intersection of a problem in $\NP$ and a problem in $\co \NP$.
On top, we use more classes from the polynomial hierarchy~\cite{StockmeyerMeyer73,Stockmeyer76,Wrathall76},
$\DeltaP0 \dfn \PiPtime0 \dfn
\SigmaPtime0 \dfn \Ptime$ and $\DeltaP{i} \dfn
\Ptime^{\SigmaPtime{i-1}}$, $\SigmaPtime{i} \dfn
\NP^{\SigmaPtime{i-1}}$, and $\PiPtime{i} \dfn
\co\NP^{\SigmaPtime{i-1}}$ for $i>0$ where $C^{D}$ is the class~$C$ of
decision problems augmented by an oracle for some complete problem in
class $D$.
The complexity class $\DP_k$ is defined as
$\DP_k\dfn \SB L_1 \cap L_2 \SM L_1\in \SigmaPtime{k}, L_2 \in
\PiPtime{k}\SE$, $\DP{=}\DP_1$~\cite{LohreyRosowski23}.

The Boolean {\em satisfiability} problem (SAT) for formulas in {\em conjunctive normal form} (CNF), asks: given $\varphi=\bigwedge_{i=1}^m C_i$ where each $C_i$ is a clause, decide whether $\varphi$ admits at least one satisfying assignment.
SAT is the one of the most famous $\NP$-complete problems, whereas its complementary problem to check {\em unsatisfiability}, is $\co \NP$-complete.
Moreover, the problem to check whether $\varphi$ is satisfiable and $\psi$ unsatisfiable for a given pair of formulas $(\varphi, \psi)$ (the SAT-UNSAT problem) is $\DP$-complete.
For $\PiP$ one usually considers the evaluation problem for a {\em quantified Boolean formula} of the form $\forall X \exists Y . \varphi$ where $X$ and $Y$ are two disjoint sets of variables and $\varphi$ a formula in CNF over $X$ and $Y$.
Finally, for $\SigmaP$ the problem is to check whether $\forall X \exists Y . \varphi$ is false.

\paragraph{Abstract Argumentation.}
We use Dung's argumentation framework~(\cite{Dung95a}) and consider only non-empty and finite sets of arguments~$A$.
An \emph{(argumentation) framework~(AF)} is a directed graph~$F=(A, R)$, where $A$ is a set of arguments and $R \subseteq A\times A$, consisting of pairs of arguments
representing direct attacks between them.
\longversion{Let the \emph{direct attack relationship} between two sets~$E,E'\subseteq A$ of arguments be defined by relations $\rightarrowtail_R$ and $\leftarrowtail_R$ as follows:
$E$ \emph{directly attacks}, i.e.,  $E\rightarrowtail_R E'\eqdef \{a\in E \mid (\{a\}\times E')\cap R \neq \emptyset\}$,
and $E\leftarrowtail_R E'\eqdef \{a\in E \mid (E' \times \{a\}) \cap R\neq \emptyset\}$. }%
Let $E\subseteq A$ and $a\in A$ be an argument.
Then we say that $a$ is \emph{defended by $E$ in $F$} if for every $(a', a) \in R$, there exists $a'' \in E$ such that $(a'', a') \in R$.
The family~$\adef_F(E)$ is defined by $\adef_F(E) \eqdef\{\; a \mid a \in A, a \text{ is defended by $E$ in $F$}  \;\}$.
In abstract argumentation, one aims at  computing so-called \emph{extensions}, which are subsets~$E \subseteq A$ of the arguments that have certain properties.
The set~$E$ of arguments is called \emph{conflict-free in~$E$} if $(E\times E) \cap R = \emptyset$; $E$ is \emph{admissible in $F$} if
(1) $E$ is \emph{conflict-free in $F$}, and
(2) every $a \in E$ is \emph{defended by $E$ in $F$}.
Let $E^+_R\eqdef E\cup\{\, a\mid (b,a)\in R, b \in E\, \}$ and $E$ be conflict-free.
Then, $E$ is (1) \emph{naive} in $F$ if no $E' \supset E$ exists that is {conflict-free in $F$}, and
(2) \emph{stage in $F$} if there is no conflict-free set~$E'\subseteq A$ in~$F$ with~$E^+_R\subsetneq (E')^+_R$.
An admissible set $E$ is
(1) \emph{complete in~$F$} if $\adef_F(E) = E$;
(2) \emph{preferred in~$F$}, if no $E' \supset E$ exists that is \emph{admissible in $F$};
(3) \emph{semi-stable in $F$} if no admissible set $E' \subseteq A$ in~$F$ with~$E^+_R\subsetneq (E')^+_R$ exists; and
(4) \emph{stable in~$F$} if every $a \in A \setminus E$ is \emph{attacked} by some $a' \in E$.
%
%
For a semantics~$\sem \in \{\cnf, \nai, \adm, \comp, \stab, \pref,\semi,\stag\}$, we write $\sem(F)$ for the set of \emph{all extensions} of semantics~$\sem$ in $F$.
%
%
%
Let~$F=(A,R)$ be an AF.
Then, the problem $\Exist\sem$ asks if $\sem(F)\neq\allowbreak\emptyset$. 
%
%
The problems $\cred{\sem}$ and $\skep{\sem}$ question for 
$a{\,\in\,}A$, whether~$a$ is in some $E\in \sem(F)$ (``\emph{credulously} accepted'') or every~$E\in\sem(F)$  (``\emph{skeptically} accepted''), respectively.
We let $\Cred\sigma$ (resp., $\Skep\sigma$) denote the set of all credulously (skeptically) accepted arguments under semantics $\sigma$.
The complexity of AF reasoning is well-known,~e.g.,
\cite{flap/DvorakD17}.

\section{Diversity in Abstract Argumentation}
In this section, we introduce a quantitative notion of \emph{diversity
of extensions} based on the symmetric-difference.

\begin{definition}
Let $F=(A,R)$ be an AF, $\sigma$ be a semantics and $k\geq 0$ be an
integer.  For two sets $S\subseteq A$ and $T\subseteq A$ of
arguments, we define by $d(S,T)$ the \emph{distance} between $S$
and~$T$.  Formally, we let
\begin{align*}
	d(S,T) &\dfn |S\symdif T|= |(S\cup T)\setminus (S\cap T)|
\end{align*}
We say that two extensions $S,T \in \sem(F)$ are \emph{$k$-diverse
	if $d(S,T)= k$}.
Moreover, for an extension~$S\in \sem(F)$ and a set~$X$, we say that
$S$ \emph{covers} $X$ if $X \subseteq S$.
\end{definition}

We continue our example from above.
\begin{example}
Consider the AF from Example~\ref{ex:running}, its extensions, and
their symmetric difference.
We can see that one extension supports treatment $X$ while another
supports $Y$ under stable semantics.  These extensions are not only
\emph{different}, but are very \emph{diverse} due to their large
symmetric difference.

The arguments ``StartX'' and ``StartY'' are both credulously accepted under
the stable semantics. 
However, asking whether
both arguments 
are $k$-diverse for some $k\in\mathbb N$
reveals how incompatible these acceptable choices actually are, or
how radically different two justified medical decisions in this setting can be.
In this context, smaller values of $k$ highlight that the two
treatments differ only marginally,~i.e., same medical reasoning, but
different preference.  
In contrast, larger values of $k$ depict two
entirely different diagnostic strategies.  Thus, in decision-making,
higher diversity indicates deeper disagreement between choices.
\end{example}

This consideration leads us to the following problem.

\pbDef{$\kdiversearg_\sem(F,K,a,b)$}{an AF $F$, $k\in\mathbb N$,
and arguments $a,b$.}{does $F$ admit two $k$-diverse
$\sem$-extensions $S$ and $T$ such that $a\in S$, $b\in T$?}

Note that we can extend $\kdiversearg_\sem$ from individual arguments to the sets of arguments, which further allows us to compute the diversity when ``absence'' (or even complements) of arguments are considered. 
%
Since obtained complexity results are identical, we did not explicitly mention these results.

We can also ask the diversity-related question at the level of an AF.
That is, does an AF have two $k$-diverse $\sem$-extensions for a given
semantics $\sem$.  Here, diversity measures how deep the disagreement
encoded in the AF is, beyond mere multiplicity of extensions. 
Our notion of diversity applies to the extension space (all acceptable sets of arguments) under some semantics. Therefore, disagreement is interpreted in terms of extensions (or accepted viewpoints).
High diversity implies that the given AF supports radically different and internally coherent positions, thus indicating higher level of disagreement and low robustness of conclusions.  On the other hand, low diversity implies that extensions differ only marginally with most reasons being shared.
Consequently, they indicate near-consensus and robust conclusions with shallower disagreements.
%
%
%
%


\pbDef{$\kdiverseext_\sem(F,K)$}{an AF $F$ and $k\in\mathbb N$}{does
$F$ admit two \textbf{non-empty} $k$-diverse $\sem$-extensions?}

%
%
%
The $\kdiverseext_\sem(F,K)$ problems 
immediately give rise to whether there are $\sem$-extension that are
at least/most $k$-diverse, which we abbreviate by $\abovekdiverseext$
and $\belowkdiverseext$, respectively. Similarly,
$\kdiversearg_\sem(F,K,a,b)$ gives rise to $\abovekdiversearg$ and
$\belowkdiversearg$, respectively.

Finally, we define the problem to obtain the maximum~$k$.
\pbDefT{$\maxkdiverseext_\sem(F,A,B)$}{an AF~$F$ and sets~$A$,$B$ of
arguments}{Output the maximum~$k$ for which $A$ and $B$ are covered
by $k$-diverse extensions.}
Similarly, we can define the corresponding minimization
task $\minkdiverseext_\sem(F,A,B)$.

\paragraph{Levels of Diversity.}
It could be that some arguments are $k$-diverse and $k'$-diverse for
some $k'< k\in\mathbb N$ without being $\ell$-diverse for any
$k'\leq \ell\leq k$. We illustrate how an AF might have different
\emph{levels} of diversity among its extensions. 

\begin{figure}[t]
\centering
\begin{subfigure}[t]{0.9\textwidth}
	\centering
	\includegraphics{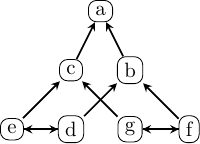}
	\caption{%
		Stable $4$ and $6$-diversity 
		and
		admissible $1$--$3$ and $5$-diversity.
	}
	\label{fig:levels-stab}
\end{subfigure}\\[0.5em]

\begin{subfigure}[t]{0.9\textwidth}
	\centering
	\includegraphics{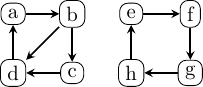}
	\caption{%
		Neither $3$-diverse nor $5$-diverse adm-extensions.
	}
	\label{fig:levels-adm}
\end{subfigure}
\caption{Example AFs showing different levels of diversity.}
\end{figure}


\begin{example}\label{ex:levels-stab}\label{ex:levels-adm}
Consider the AF $F_1=(A,R)$ as depicted in
Figure~\ref{fig:levels-stab}.
$F_1$ admits the following stable sets
$E_1=\{a,e,f\}, E_2=\{b,e,g\}, E_3=\{a,d,g\}$, and $E_4=\{c,d,f\}$.
The pair-wise symmetric difference of these sets yield diversity
values in $\{4,6\}$.  In particular $F_1$ does not admit $k$-diverse
stable extensions for any $k\in \{1,2,3,5\}$.
When considering admissible semantics, $F_1$ additionally admits the
following singleton $\{e\}$, $\{d\}$, $\{g\}$, $\{f\}$, and
two-element extensions $\{e,g\}$, $\{e,f\}$, $\{d,g\}$, and
$\{d,f\}$.  Here, it can be seen that under admissible semantics,
$F_1$ also admits diversity values of $1$ (due to extensions
$\{e\}$, $\{e,f\}$), $2$ (due to $\{e\},\{d\}$), and $3$ (due to
$\{e\}$, $\{d,f\}$).  Finally, by considering sets $\{a,e,f\}$ and
$\{d,g\}$, we also obtain two $5$-diverse $\adm$-extension.

Next, consider the AF $F_2=(A,R)$ as depicted in
Figure~\ref{fig:levels-adm}.  The AF~$F_2$ admits the 
admissible sets
$E_1=\{a,c\}$, $E_2 = \{e, g\}$, $E_3 = \{f,h\}$, $E_4 =\{a,c,f,h\}$, and
$E_5=\{a,c,e,g\}$.  The pair-wise symmetric difference of these sets
yield diversity values in $\{2,4,6\}$. Notable, there are neither
$3$-diverse nor $5$-diverse $\adm$-extensions.
\end{example}

\section{Complexity of Reasoning about Diversity}

We first consider the complexity of deciding whether an AF admits two $k$-diverse extensions ($\kdiverseext_\sem$) for a semantics $\sem$, or whether additionally contains two arguments in the question ($\kdiversearg_\sem$). 
Moreover, we also consider the complexity of deciding whether an AF admits some diversity level above, or below 
$k$, i.e., $\abovekdiversearg_\sem$ and $ \belowkdiversearg_\sem$, respectively.
Finally, we consider analogous problems ($\abovekdiverseext_\sem$ and $ \belowkdiverseext_\sem$) for extensions.

We find it convenient to present our hardness reductions on top of the \emph{standard translation} of SAT/QBF instances to AFs.
The following reduction is relevant for some hardness results.
\begin{definition}[\cite{flap/DvorakD17}]\label{def:translation}
Let $\varphi\dfn\bigwedge_{i=1}^m C_i$, be a CNF-formula where each $C_i$ is a clause.
Consider the AF $F_\varphi=(A,R)$ constructed as follows:
\begin{align*}
	A\dfn &\{\varphi, C_1,\dots,C_m\}\cup\{\,x,\bar x\mid x\in\var(\varphi)\} \\
	R \dfn &\{\,(C_i,\varphi)\mid i\leq m\,\}\cup \{\,(x,\bar x),(\bar x,x)\mid x\in\var(\varphi)\,\}\\
	&\cup\{\,(x,C_i)\mid x\in C_i\,\}\cup\{\,(\bar x,C_i)\mid \bar x\in C_i\,\}.
\end{align*}
We call $F_\varphi$ the argumentation framework of  $\varphi$ {generated} via the \emph{standard translation} for admissible semantics.
\end{definition}
Then, $\varphi$ is satisfiable iff the argument $\varphi$ is credulously accepted in $F_\varphi$ under semantics $\sigma\in\{\adm,\comp,\stab\}$.


\subsubsection{Above, Below, and Exact $k$ Diversity}%
The problem $\kdiversearg_\sem(F,k,a,b)$ asks: given two arguments $a$ and $b$ in an AF $F=(A,R)$ and a number $k\in\mathbb N$, are there two $\sem$-extensions $S_1,S_2$ with $a\in S_1, b\in S_2$, and $d(S_1,S_2)= k$?
We prove that this problem (i) shows higher complexity than credulous reasoning for conflict-free, naive, and preferred semantics, whereas (ii) is of same complexity as credulous reasoning under all other considered semantics.

We also consider the problem $\kdiverseext_\sem(F,k,a,b)$ asking: given an AF $F=(A,R)$ and a number $k\in\mathbb N$, are there two $\sem$-extensions $S_1,S_2$ with  $d(S_1,S_2)= k$?

\subsection{Conflict-free and Naive Semantics}\label{sec:conf}
We first consider the \emph{lightest} semantics.
Surprisingly, here we get hardness already for the conflict-free semantics.
This case is interesting as almost all other reasoning problems (extension existence, credulous and skeptical acceptance, etc.) are decidable in polynomial time under conflict-free semantics.
Nevertheless, the hardness is more involved in some cases (e.g., for below $k$ diversity) and in fact both semantics incur different complexity in those cases (under reasonable complexity assumptions).


\begin{restatable}[$\star$]{theorem}{confatleastdivarg}\label{thm:conf-atleast-div-arg}
The problems $\kdiversearg_\sem$ and $\abovekdiversearg_\sem$ are $\NP$-complete for $\sem\in\{\cnf, \nai\}$. 
\end{restatable}

\begin{proof}
For membership, one guesses two witness sets $S_a, S_b\subseteq A$ such that (i) $S_a$ and $S_b$ are $\sem$-extension, (ii) $a\in S_a$, $b\in S_b$, and  (iii) $d(S_a,S_b)= k$.
The verification requires polynomial time for each $\sem\in\{\cnf,\nai\}$. 
For naive, one checks conflict-freeness and additionally that adding any argument not in the set breaks the conflict-freeness.
For $\abovekdiversearg_\sem$, we substitute (iii) by $d(S_a,S_b)\geq k$.

For hardness, we reduce from the canonical $\NP$-complete problem $3\BSAT$.
Here, the reduction in Definition~\ref{def:translation} can not be directly used as it is based on the notion of admissibility, which both semantics $\sem\in\{\cnf,\nai\}$ lack.
Our reduction builds on the idea of the standard reduction for proving the hardness of independent sets existence (IS) problem implicitly.
However, one can not directly reduce from IS as we do not have control over the size of solution sets.

Let $\varphi = \bigwedge_{i\leq m} C_i$ where each $C_i= (\ell_{i1}\lor \ell_{i2}\lor \ell_{i3})$ is a clause with three literals over the set $X$ of variables.
In our AF, we only require  symmetric edges. Hence we find it convenient to write $\{u,v\}$ for two attacks $(u,v)$ and $(v,u)$.
We construct an AF $F=(A,R)$ where:
\begin{align*}
	A\dfn & \{\ell_{ij} \mid 1\leq i\leq m, 1\leq j\leq 3 \} \cup \{a,b\},\\
	R \dfn & \{\{\ell_{i1}, \ell_{i2}\}, \{\ell_{i2}, \ell_{i3}\}, \{\ell_{i3}, \ell_{i1}\} \mid 1\leq i\leq m\} \\
	& \cup \{\{\ell,\ell'\} \mid  \ell=x, \ell'=\bar x \text{ for some } x\in X\} \\
	& \cup \{\{a,b\}\} \cup \{\{b,x\} \mid x\in A\setminus\{b\}\},
\end{align*}
for fresh arguments $a,b\not\in X$.
Moreover, we let $k=m+2$.
Here, we observe that $\{b\}$ is a naive and hence also a conflict-free set in $F$.
Intuitively, we let each position of a literal in a clause as an argument thus resulting in three arguments for each clause of $\varphi$.
For attacks, we add a ``triangle'' of symmetric attacks between all three arguments for a clause, and further attacks between two literals of opposite parity.
For correctness, we prove the following claim.

\begin{restatable}[$\star$]{claim}{claimsatnaive}\label{claim:sat-naive}
	Arguments $a$ and $b$ are $(m+2)$-diverse in $F$ iff formula $\varphi$ is satisfiable.
\end{restatable}
We conclude by observing that the presented reduction also applies to the problem that asks above $k$ diverse extensions covering $a$ and $b$, that is, to $\abovekdiversearg_\sem$.
\end{proof}


Next, we consider the problem of (general) $k$-diverse extensions in an AF (hence, to determine whether an AF admits two at least $k$-diverse extensions).
The reduction in the proof above does not apply directly as $F$ may have two $\sem$ extensions only coming from the arguments in triangles.
This has the consequence that $\varphi$ might not be satisfiable but still $F$ has two $k$-diverse extensions.
Luckily, we can still utilize the ideas and expand our reduction to obtain the following.

\begin{theorem}\label{thm:conf-atleast-div-ext}
The problems $\kdiverseext_\sem$ and $\abovekdiverseext_\sem$ are $\NP$-complete for $\sem\in\{\cnf, \nai\}$. 
\end{theorem}

\begin{proof}
For membership, guess two witness sets as before.

For hardness, we reuse the reduction from the proof of Theorem~\ref{thm:conf-atleast-div-arg}.
The intuition why the same reduction may fail is the following.
An AF $F$ might have two small conflict-free/naive sets such that they give jointly a diversity of $k$. 
However, we are looking for one satisfying assignment for all the $m$ clause in $\varphi$.
To overcome this issue, we duplicate each argument in $F$ and ask for a larger diversity value instead. 

Let $F=(A,R)$ be the AF constructed in the proof of Theorem~\ref{thm:conf-atleast-div-arg} for a formula $\varphi$ with $m$ clauses.
Here, we do not need arguments $a$ and $b$, hence our arguments set $A$ and attacks in $R$ does not include these auxiliary arguments.
As before, we only consider symmetric edges in our AF and write $\{u,v\}$ for attacks. 
We consider an AF $F'=(A',R')$ where
\begin{align*}
	A' \dfn & A\cup \{s_* \mid s\in A\}, \\
	R'\dfn & R \cup \{\{s,s_*\} \mid s\in A\} \\
	& \cup \{\{s,t_*\}, \{s_*,t_*\}, \{s_*,t\}\mid \{s,t\}\in R\}.
\end{align*}
Intuitively, each argument $s$ now has a \emph{shadow} vertex $s_*$ which participates in all the edges as $s$.
Additionally, $s_*$ also have edges to the shadow vertices of each neighbor $t$ of $s$ in $F$.
This encodes that whenever a set $S\subseteq A$ is conflict-free in $F'$, then its shadow set $S_*= \{s_* \mid s\in S\}$ is also conflict-free whereas neither $S\cup \{t_*\}$ nor $S_*\cup\{t\}$ is conflict-free for any $t\in S$ and its shadow $t_*\in S_*$, respectively.
Moreover, in this case, the set $S\setminus\{t\}\cup \{t_*\}$ and $S_*\setminus\{t_*\}\cup\{t\}$ are both still conflict-free in $F'$.
We first prove this claim in the following.
For brevity, we denote $\ell_*=a_*$ if $\ell=a$ and $\ell_*=a$ for $\ell=a_*$ (intuitively, the shadow of a shadow argument is the argument).
\begin{restatable}[$\star$]{claim}{claimshadow}\label{claim:shadow}
	(I) $S$ is conflict-free (naive) in $F'$ iff $S_*$ is conflict-free (naive) in $F'$ for any $S\subseteq A$.
	(II) For any conflict-free set $L\subseteq A'$ and $\ell\in L$: the sets $(L\setminus \{\ell\})\cup \{\ell_*\}$ is also conflict-free.
\end{restatable}
We next prove the correctness via the following claim.
\begin{restatable}[$\star$]{claim}{claimcorrcnfnaive}\label{claim:corrcnfnaive}
	$\varphi$ is satisfiable if and only if $F$ admits two $\sem$-extensions which are $2m$-diverse for $\sem\in\{\cnf, \nai\}$.
\end{restatable}
We conclude by observing that the reduction also applies to $\abovekdiverseext_\sem$ for both semantics $\sem\in\{cnf,\nai\}$.
\end{proof}

We next turn to the problem of asking whether two arguments are \emph{below} $k$ diverse.
Here, the hardness holds only for naive semantics as we establish next.
Observe that the reduction from before does not apply here as in Claim 3: a ``small'' naive set in $F$ may yield two above-$k$-diverse extensions for $a$ and $b$ while the corresponding literals still do not satisfy all the clauses (e.g., when a witness naive set for $a$ has size $<m$).
In fact, this also explains why the hardness does not apply to the conflict-free semantics.

\begin{theorem}\label{thm:nai-atmost-div-arg}
The problem $\belowkdiversearg_\nai$ is $\NP$-complete.
\end{theorem}
\smallskip\noindent\emph{Proof. } 
For membership, guess two witness sets as before.

For hardness, we reduce from the dominating set problem.
Given a graph $G=(V,E)$ and an integer $k\in\mathbb N$: the question asks whether there are $k$ vertices $S\subseteq V$ such that for every vertex $v\in V$, either $v\in S$ or there exists $u\in S$ with $\{u,v\}\in E$.
We consider an AF $F=(A,R)$ with symmetric attacks where, for fresh $a,b\not\in V$: \qquad $A\dfn V\cup \{a,b\},$
\begin{align*}
R\dfn \{\{x,y\} \mid \{x,y\}\in E\} \cup \{\{a,b\}, \{b,x\} \mid x\in V\}.
\end{align*}%
Then, we prove the following claim:%
\begin{restatable}[$\star$]{claim}{claimdom}\label{claim:dom}
$G$ admits a dominating set of size $k$ iff arguments $a$, $b$ are below $(k+2)$-diverse in $F$ (naive semantics). \qed
\end{restatable}
%
\future{
\alert{The reduction does not work for general extensions and the previous idea does not work either, so
	$\belowkdiverseext_\nai$ is open.}
}

The reduction 
does not work for conflict-free semantics, but 
turns out to be trivial as any two non-self-attacking arguments are trivially below $k$-diverse for any $2\leq k\in \mathbb N$. 

\begin{remark}\label{rem:conf-atmost-div-arg}
Let $F=(A,R)$ be an AF and $a,b\in A$ be two arguments such that none of these arguments have a self-attack. Then, considering conflict-free semantics 
%
	$a$ and $b$ are at least ($\geq$) $2$-diverse iff either $(a,b)\in R$ or $(b,a)\in R$.
	Thus, any two arguments in an AF (without self-attacks) are trivially (below) $k$-diverse for any $k\neq 1$.
	The answer for $k=1$ depends on whether two arguments are conflicting.
	{In fact, no two conflicting arguments can be $1$-diverse in any AF}. 
\end{remark}

\subsection{Admissible, Complete, and Stable Semantics}\label{sec:adm}
We now consider semantics $\sem\in\{\adm,\comp,\stab\}$ and prove that the complexities are similar to credulous reasoning.
That is, reasoning on diverse arguments and extensions does not incur additional complexity compared to argument acceptance.%
\begin{restatable}[$\star$]{theorem}{thmdivarg}\label{thm:div-arg}
	Problems $\kdiversearg_\sem, \abovekdiversearg_\sem$, and $\belowkdiversearg_\sem$ are $\NP$-complete for $\sem\in \{\adm,\comp,\stab\}$.
\end{restatable}
Details are given in the appendix, as we mostly adapt standard reductions. We next establish that the problem with diverse extensions admits the same complexity.

\begin{restatable}[$\star$]{theorem}{thmdivext}\label{thm:div-ext}
	Problems $\kdiverseext_\sem, \abovekdiverseext_\sem$, and $\belowkdiverseext_\sem$ are $\NP$-complete for $\sem\in \{\adm,\comp,\stab\}$.
\end{restatable}

$\NP$-hardness for preferred semantics follows from admissible semantics. However, we do not have membership, as preferred extensions cannot be verified in polynomial time.

\subsection{Semi-stable, Stage, and Preferred Semantics}\label{sec:semi}
The following lemma is essential for proving claims regarding semi-stable and stage semantics.
Let $F=(A,R)$ and $a,b\in A$ be two arguments.
We say that $a$ and $b$ are \emph{$R$-equivalent} in $F$ if they have identical attacks in $F$, i.e., if $(a,c)\in R \iff (b,c)\in R$ and $(c,a)\in R \iff (c,b)\in R$ for any $c\in A$.

\begin{restatable}[$\star$]{lemma}{lemrange}\label{lem:range}
	Let $F=(A,R)$ be an AF and $a,b\in A$ be two $R$-equivalent arguments in $F$; for any $S\subseteq A$ with $a\in S$:
	\begin{enumerate}
		\item $S$ is conflict-free (admissible) in $F$ iff $(S\setminus \{a\}) \cup \{b\}$ is conflict-free (admissible) in $F$.
		\item if $(a,b)$ and $(b,a)\in R$, then  $S^+_R = S'^+_R$ for $S'=(S\setminus \{a\}) \cup \{b\}$.
	\end{enumerate}
\end{restatable}

\begin{theorem}\label{thm:div-arg-semi}
	The problems $\kdiversearg_\sem$, $\abovekdiversearg_\sem$ and $\belowkdiversearg_\sem$ are  $\SigmaP$-complete for $\sem\in \{\semi,\stag\}$.
\end{theorem}
\begin{proof}
	For membership, observe that one can guess two sets of arguments $S_a,S_b\subseteq A$ and verify that both are $\sem$-extensions, $k$-diverse, and contain $a$ and $b$, respectively. 
	The verification for $\sem$-extensions requires an $\NP$ oracle to decide range-maximality.
	%
	%
	%
	%
	Moreover, the same also applies to extensions with diversity $\leq k$ and $\geq k$.
	
	We next turn to prove hardness. 
	To this aim, we utilize an existing reduction by~\cite{DVORAK2010425} proving credulous acceptance under both considered semantics.
	Given a QBF $\Phi = \forall Y \exists Z.\varphi$ where $\varphi\dfn\bigwedge_{i=1}^m C_i$ is a CNF-formula with clauses $C_i$ over variables $X = Y \cup Z$. 
	We let $Y_*=\{y_*\mid y\in Y\}$ and $\bar Y_*=\{\bar y_*\mid y\in Y\}$ as auxiliary sets of arguments and consider the AF $F_\Phi=(A,R)$ constructed as follows:
	\begin{align*}
		A \dfn & \{\varphi,\bar \varphi, b\} \cup \{C_1 , \dots, C_m\} \cup X \cup \bar X \cup Y_* \cup \bar Y_*\\
		R \dfn &\{\,(C_i,\varphi)\mid i\leq m\,\} \cup \{\,(x,\bar x),(\bar x,x)\mid x\in X\,\}\\
		&\cup\{\,(x,C_i)\mid x\in C_i\,\}\cup\{\,(\bar x,C_i)\mid \bar x\in C_i\,\} \\
		& \cup \{(\varphi, \bar\varphi), (\bar\varphi, \varphi), (\varphi, b), (b, b)\,\}\\
		& \cup \{(y,y_*), (\bar y, \bar y_*) ,(y_*,y_*), (\bar y_*,\bar y_*)\mid y\in Y\}.
	\end{align*}
	It is known that $\Phi$ is true iff $\varphi$ is contained in each semi-stable (stage) extension of $F_\Phi$ iff $\bar\varphi$ is not contained in any semi-stable (stage) extension of $F_\Phi$.
	To complete the reduction, we construct an AF $F'=(A',R')$ by taking a fresh argument $\bar\varphi_c\not\in A$, such that $A'\dfn A\cup \{\bar\varphi_c\},$
	\begin{align*}
		R'\dfn R \cup \{(\varphi,\bar\varphi_c), (\bar\varphi_c,\varphi), (\bar\varphi,\bar\varphi_c),(\bar\varphi_c,\bar\varphi)\}.
	\end{align*}
	Intuitively, the argument $\bar\varphi_c$ serves as a copy of $\bar\varphi$. Whereby, it is accepted in some $\sem$-extension iff $\bar\varphi$ is accepted in some $\sem$-extension for $\sem\in\{\semi,\stag\}$ due to Lemma~\ref{lem:range}.
	We prove correctness of our reduction in the following claim.
	\begin{restatable}[$\star$]{claim}{claimsemisigmap}\label{claim:semi-sigmap}
		$\Phi$ is false iff argument $\bar\varphi$ is contained in some $\sem$-extension of $F_\Phi$ for $\sem\in\{\semi,\stag\}$ iff arguments $\bar\varphi$ and $\bar\varphi_c$ are $2$-diverse in $F'$ under $\sem$.
	\end{restatable}
	
	Note that the reduction in the proof above also applies to the case of ${\leq k}$- and ${\geq k}$-diverse arguments for $\sem\in\{\semi,\stag\}$.
	This holds as one can still take $k=2$.
\end{proof}

For deciding diverse extensions, the reduction in the proof of Theorem~\ref{thm:div-arg-semi} does not directly apply for $\sem\in\{\semi,\stag\}$.
This holds due to the following issue: when the CNF $\varphi$ in $\Phi$ is satisfiable, $F_\Phi$ already admits a $\sem$-extension of size $|X|+1$ where  $X= Y\cup Z$ and hence $|X|>|Y|$.
It might happen that $\Phi$ admits two satisfying assignments which only differ at one literal.
Thus, their corresponding $\sem$-extensions also differ at only one argument and hence are $2$-diverse.
Nevertheless, we can expand our reduction to consider sufficiently many copies of the argument $\bar\varphi$ in $F$ to avoid this case. The proof of the following theorem formalizes this intuition.
Interestingly, this trick with multiple copies does not work in the case of $\belowkdiverseext_\sem$.
This holds since smaller than $k$ diversity is already achievable via any satisfying assignments for $\varphi$.

\begin{restatable}[$\star$]{theorem}{thmdivextsemi}\label{thm:div-ext-semi}
	The problems $\kdiverseext_\sem$ and $\abovekdiverseext_\sem$  are $\SigmaP$-complete for semantics $\sem\in \{\semi,\stag\}$. 
\end{restatable}

We next turn towards preferred semantics which possesses the same complexity, modulo a different reduction.
\begin{restatable}[$\star$]{theorem}{divargpref}\label{thm:div-arg-pref}
	The problem $\kdiversearg_\pref$ and $\belowkdiversearg_\pref$ are both $\SigmaP$-complete.
\end{restatable}

However, the precise complexity of $\abovekdiversearg_\pref$ is open. 
The issue why the above reduction does not work for $\abovekdiversearg_\pref$ is the following.
If $\varphi$ is satisfiable, one can construct an extension containing $\bar b$ from a satisfying assignment, of size $|X|+2$, which together with $\{b\}$ gives two $(|X|+3)$-diverse extensions which is above $k$ anyway.

We also extablish the complexity for $k$-diverse extensions. 
\begin{restatable}[$\star$]{theorem}{thmdivextpref}\label{thm:div-ext-pref}
	The problem $\kdiverseext_\pref$ is $\SigmaP$-complete. 
\end{restatable}
\future{As one would expect, the reasoning in the reduction of Theorem~\ref{thm:div-ext-pref} does not extend to $\belowkdiverseext_\pref$ and $\abovekdiverseext_\pref$ due to its concrete nature.
	\yasir{Therefore, the precise complexity for $\belowkdiverseext_\pref$ and $\abovekdiverseext_\pref$ remain open. Nevertheless, we have membership in $\SigmaP$. 
	}
}
\subsection{Maximal Diversity}\label{sec:max}
Finally, we compute the threshold for maximum diversity.
Precisely, we ask whether an AF has $k$-diverse $\sem$-extensions (for a pair of arguments) but not $k'$-diverse for any $k'{>}k$.

We reduce SAT-UNSAT (for QBF) to the considered semantics by encoding two formulas into one AF and choosing a diversity threshold $K$ that separates them.
The construction ensures that the AF has two $k$-diverse extensions iff the first formula is satisfiable (true), and no $k'$-diverse extensions for $k'>k$ iff the second formula is unsatisfiable (false).

\begin{restatable}[$\star$]{theorem}{thmconfmaxkdiv}\label{thm:conf-max-k-div}
	The problem $\maxkdiverseext_\sem$ is $\DP$-complete for each $\sem\in\{\cnf, \nai\}$.
\end{restatable}

We next prove that the argument version of the diversity problem is also $\DP$-complete for $\sem\in\{\cnf,\nai\}$.
Our construction works similar as in the proof of Theorem~\ref{thm:conf-max-k-div}.
\begin{restatable}[$\star$]{theorem}{thmmaxext}\label{thm:conf-max-k-div-arg}
	The problem $\maxkdiversearg_\sem$ is $\DP$-complete for each $\sem\in\{\cnf, \nai\}$.
\end{restatable}

Next, we consider the connection between extensions. 

\begin{restatable}[$\star$]{lemma}{lemunion}\label{lem:union}
	Let $F_1=(A_1,R_1)$, $F_2=(A_2,R_2)$ be two mutually disjoint AFs, and let $F=(A_1\cup A_2, R_1\cup R_2)$. 
	Then, for each semantics $\sem$ and sets $S_i\subseteq A_i$:
	$S_1\cup S_2$ is a $\sem$-extension in $F$ iff $S_i$ is a $\sem$-extension in $F_i$ for $i=1,2$.
\end{restatable}

By Lemma~\ref{lem:union}, when considering disjoint AFs, their extensions are sub-extensions in the individual sub-AFs.

We next establish that the argument version of the diversity problem is $\DP$-complete for $\sem\in\{\adm,\comp,\stab\}$.
\begin{restatable}[$\star$]{theorem}{maxext}\label{thm:max-k-div-adm}
	The problem $\maxkdiversearg_\sem$ is $\DP$-complete for each $\sem\in\{\adm,\comp,\stab\}$.
\end{restatable}
Also the extension-version of the problem is $\DP$ and $\DPtwo$-c.




\begin{restatable}[$\star$]{theorem}{thmmaxkdiv}\label{thm:max-k-div}
	The problem $\maxkdiverseext_\sem$ is
	\begin{enumerate}
		\item $\DP$-complete for each $\sem\in\{\adm,\comp,\stab\}$ and
		\item  $\DPtwo$-complete for semantics $\sem\in \{\semi,\stag\}$.
	\end{enumerate}
\end{restatable}



\paragraph{A Note On Minimal Diversity.}
Since diversity is neither monotonic, nor monotonic, the problem of determining the smallest $k$ for which an AF admits two $k$-diverse $\sem$-extensions is not complementary to its maximum counterpart.
Moreover, our hardness-reductions in Section 4.4 does not immediately extend to the case of minimal diversity.

\section{Experiments}
To investigate practical feasibility of diversity, we
implemented our approach and conducted experiments.
%

\subsection{Implementation}
We base our implementation on an ASP-encoding and the argumentation
solver
ASPARTIX~\cite{EglyGagglWoltran08,EglyGagglWoltran10,DvorakKonigWallner21}
for rapid prototyping.
Our encoding computes, given an AF~$F$, two extensions that are the
minimally or maximally~$k$-diverse in~$F$.
The core computational component of our approach is an ASP encoding
implementing the proposed diversity computation, which we solve using
clingo (5.8.0).
We place a simple 
wrapper to orchestrate the experiments,
and collect runtime and optimization statistics.
For each instance and semantics, we record the runtime required to
solve the corresponding optimization problem, as well as the resulting
diversity value~$k$.  We consider both the maximal and minimal
diversity variants.


\subsection{Instances and Setup}
As benchmark data, we use AFs from the ICCMA 2025 benchmark suite. For
each benchmark instance, we evaluate the proposed diversity
computation under the admissible, complete, and stable semantics.
Each run was subject to a fixed timeout of 300 seconds. Our analysis
focuses on instances for which an optimal solution was found within
the given timeout. To avoid the influence of a small number of
extremely large AFs, we further restrict the evaluation to AFs with at
most 1000 arguments, leaving us with 270 instances.
We conduct the experiments on a desktop machine equipped with an Intel
i9-11900K CPU running Windows 10, which is sufficient as we are
primarily interested in obtaining the maximum $k$ and minimum~$k$ for
the considered frameworks.

\begin{table}[t]
	\centering
		\begin{tabular}{l|ccc|ccc}
			\hline
			& \multicolumn{3}{c|}{Maximum diversity ($k_{\text{max}}$)} & \multicolumn{3}{c}{Minimum diversity ($k_{\text{min}}$)}\\
			Semantics  & $\mathrm{avg}(k)$ & $\mathrm{med}(k)$ & $\mathrm{avg}(t)$ (s) & $\mathrm{avg}(k)$ & $\mathrm{med}(k)$ & $\mathrm{avg}(t)$ (s) \\
			\hline
			$\adm$      & 102.352 & 71.0 & 1.591 & 1.0 & 1.0 & 0.91\\
			$\comp$     & 88.683  & 77.0 & 1.359 & 6.894 & 1.0 & 0.907\\
			$\stab$     & 101.5   & 93.0 & 0.911 & 12.478 & 2.0 & 0.587\\
			\hline
		\end{tabular}
	\caption{Summary of experimental results for AFs with at most 1000 arguments. For each semantics, we report the average and median diversity values and the average runtime (in seconds) for both maximal ($k_{\max}$) and minimal ($k_{\min}$) variants.}
	\label{tab:summary_1000}
\end{table}

\subsection{Results} In total, we evaluated 270 instances, each under
three semantics. Among these, 145 instances reached optimality within
the timeout under at least one semantics, while the remaining
instances either where unsatisfiable or exceeded the time limit
under all semantics. Table~\ref{tab:summary_1000} provides a summary
of the results aggregated over all instances in this range, reporting
the average and median diversity values, as well as the average
runtime per semantics, for both maximal and minimal variants.
To analyze how the minimum and maximum diversity distributes, we
provide Figure~\ref{fig:exfig} that plots per instance (X-axis) the
min (orange), max (blue), and median (dashed green) $k$-diversity
value (Y-axis).

\begin{figure}[t]
	\centering \includegraphics[width=.75\textwidth]{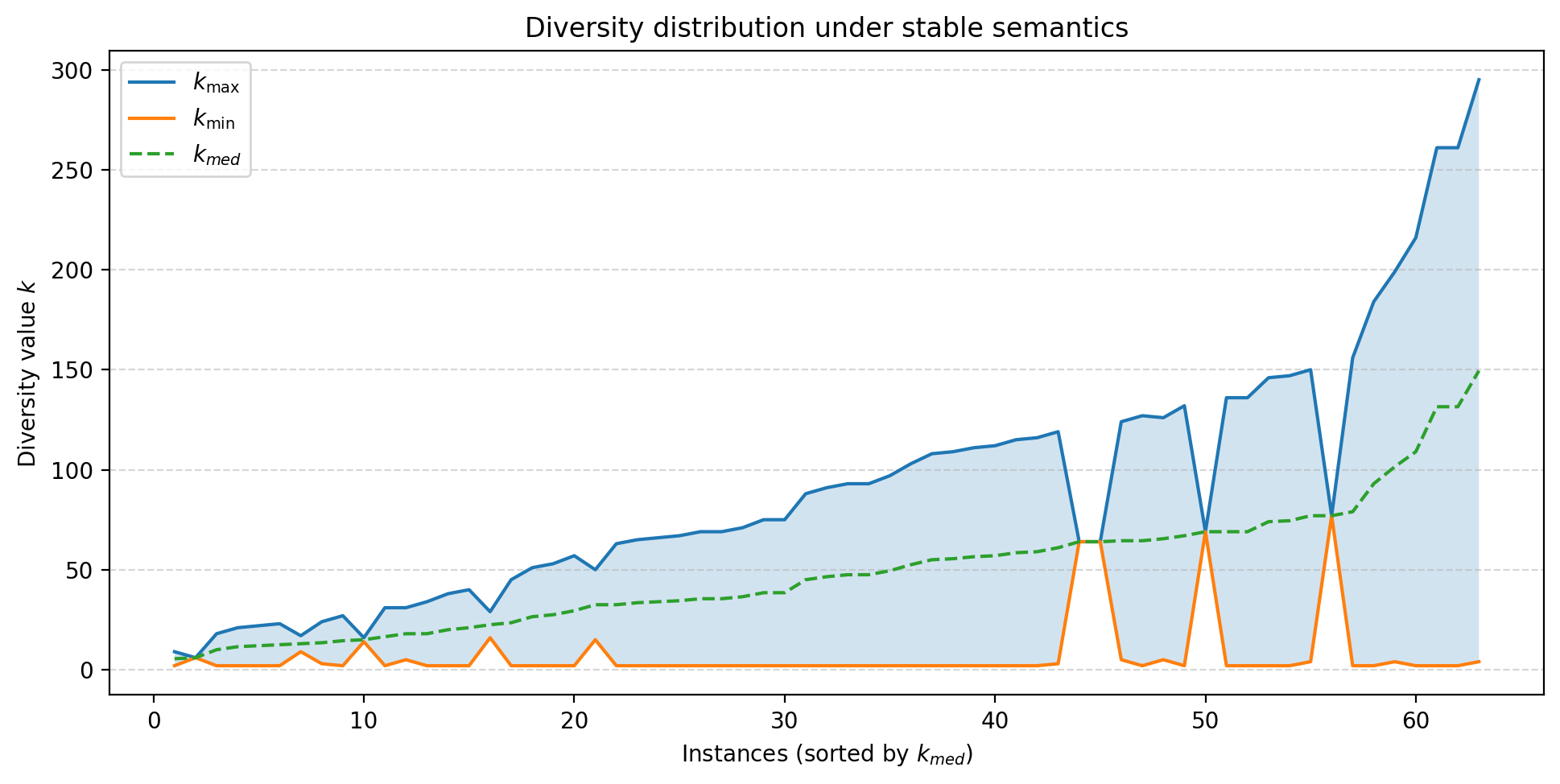}
	\caption{Distribution of min, med, and max diversity for instances
		under stable semantics. The instances are sorted by median
		diversity.}
	\label{fig:exfig}
\end{figure}

\subsection{Discussion}
The experimental results reveal characteristic patterns in how diversity values are distributed under the considered semantics.
Figure~\ref{fig:exfig} shows that, diversity values vary substantially across instances, indicating that the structure of extensions can
differ markedly even within the same AF.
In particular, the instance-wise distribution reveals that for a large
portion of instances, extensions remain relatively close to each
other, while a smaller set of instances admits extensions that are
structurally far apart.
This suggests that high diversity is not uniformly present, but rather
concentrated in specific frameworks that allow for fundamentally
different solutions.
Additional instance-wise diversity distributions under admissible and
complete semantics are provided in the appendix.

\section{Conclusion}
We introduce a quantitative notion of \emph{extension
	diversity} for abstract argumentation based on symmetric difference,
which makes explicit how strongly alternative, internally coherent
viewpoints in an AF diverge.
%
%
%
We provide a \emph{comprehensive complexity classification} for reasoning about
existence of $k$-diverse extensions, diversity covering specific
arguments, and optimization variants. 
We implement a $k$-diversity
computation 
and conduct an initial experimental
study, 
which indicates practical feasibility 
for medium-sized frameworks.

For future work, we 
are interested in completing the complexity landscape for open cases (e.g., for minimal $k$-diversity). 
To tame the high complexity, one could consider parameterized complexity or structure-aware encoding into Boolean satisfiability~\cite{FichteH0M21,MahmoodHGF26}.
%
Moreover, we propose diversity as a useful measure to consistent query answering in data- or knowledge-bases, following the natural connections to argumentation~\cite{MahmoodVBN24,0002HN25}.
Finally, extending the diversity notion to logic-based argumentation~\cite{0002M023,Hecher0M024} is another venue for future work.
%
%
\section*{Acknowledgments}
%
This research was funded in whole or in part by
the  Deutsche Forschungsgemeinschaft (DFG, German Research Foundation), grant TRR 318/3 2026 – 438445824 and the Ministry of Culture and Science of North Rhine-Westphalia (MKW NRW) within projects WHALE (LFN 1-04) funded under the Lamarr Fellow Network programme and project SAIL, grant NW21-059D,
the Austrian Science Fund (FWF) grant J4656, the French Agence
nationale de la recherche (ANR) grant ANR-25-CE23-7647; and 
Excellence Center at Linköping -- Lund in Information Technology
(ELLIIT) funded by the Swedish government and the Wallenberg AI,
Autonomous Systems and Software Program (WASP) funded by the Knut and
Alice Wallenberg Foundation.

\bibliography{main}


\appendix
\cleardoublepage

\section{Complete Proof Details}

\subsection{Proofs of Section~\ref{sec:conf}}

\claimsatnaive*
\begin{proof}[Proof of Claim]
Suppose $\varphi$ is satisfiable and let $\theta$ be a satisfying assignment, seen as a set of literals over $X$.
We let $S_\theta= \{\ell_{ij}\mid 1\leq i\leq m, 1\leq j\leq 3, \ell_{ij}\in C_i\cap \theta \}$ be a set of arguments.
Then, the set $S_\theta$ has size $m$ as each clause is satisfied and $S$ is conflict-free since $\theta$ is an assignment.
We let $S=S_\theta\cup\{a\}$ and $B=\{b\}$ be two sets.
Observe that $S$ is a naive extension in $F$ as no other argument from $A\setminus S$ can be added to it.
We conclude by observing that $d(S,B)=m+2$ since two sets do not share any argument.

Conversely, suppose $a$ and $b$ are $k$-diverse in $F$. 
The only naive set in $F$ containing $b$ is $B=\{b\}$, hence a witness for $b$ must be the set $B$ itself and it can not contain any other argument from $A\setminus\{b\}$ since $b$ attacks all those arguments.
By our claim, there must be a witness set $S$ for $a$ of size $m+1$.
Here, $S$ must also contain exactly one argument from each ``triangle'' to attain a size of $m$.
Therefore, there is a set $L$ of arguments $\ell_{ij} \in A$ for each $i\leq m$ and $j\leq 3$ such that $S=L\cup \{a\}$.
This yields a satisfying assignment $L$ for $\varphi$ as we prove next. 
First, $L$ does not contain two literals of opposite parity since set $S$ is conflict-free. Therefore $L$ is a valid assignment.
Secondly, $L$ contains one argument (or literal) for each $i\leq m$ (clause $C_i$) and therefore $L\models \varphi$.
\end{proof}

\claimshadow*
\begin{proof}[Proof of Claim]
\textbf{For (I)}, let $S\subseteq A$ and consider the set $S_*=\{s_*\mid s\in S\}$.
The only attacks for arguments in $S_*$ are of the form $\{s_*,t_*\}$ by definition of $R$. Therefore, for any $s_*,t_*\in S_*$: $\{s_*,t_*\}\in R$ iff $S_*$ is not conflict-free iff $\{s,t\}\in R$ (by definition) iff $S$ is not conflict-free.
\textbf{For naive semantics}, suppose $S$ is naive.
Then, $S_*$ is conflict-free due to the previous claim as $S$ is conflict-free.
Now, to prove that $S_*$ is naive, we let $\ell\in A'\setminus S_*$.
We consider the following two cases based on whether $\ell\in A$ or not.
(i) $\ell\in A$: we again have two sub-cases, 
(i-a) if $\ell\in S$ then $\ell_*\in S_*$ and hence $\{\ell,\ell_*\}\in R$ implies that $S_*\cup\{\ell\}$ is not conflict-free.
However, (i-b) if $\ell\not\in S$ then there is some $s\in S$ with $\{s,\ell\}\in R$ as otherwise $S$ would not be naive.
Then, $s_*\in S_*$ and hence $\{s_*,\ell\}\in R$, consequently $S_*\cup\{\ell\}$ is not conflict-free.
(ii) $\ell\in A'\setminus A$: due to $\ell\not \in S_*$, we have that $\{\ell,s\}\in R$ for some $s\in S$ as $S$ is naive. As a result, $\{\ell,s_*\}\in R$ by definition of $R$. Hence $S_*\cup\{\ell\}$ is not conflict-free.

From (i) and (ii), we conclude that $S_*$ is also naive. 
By analogous reasoning, we prove that if $S_*$ is naive then $S$ must also be naive, thus completing the proof to our first claim.

\textbf{For (II)}, let $L\subseteq A'$ be a conflict-free set and $\ell\in L$.
Then $L\setminus \{\ell\}$ is still-conflict-free.
Now, consider the set $L'= (L\setminus \{\ell\})\cup \{\ell_*\}$ and suppose $L'$ is not conflict-free.
Since $L\setminus \{\ell\}$ is conflict-free, any conflict over arguments in $L'$ must involve the argument $\ell_*$.
By definition, this involves the attack $\{\ell_*,r\}$ for some $r\in L$ such that  $r\in  \{t, t_*\}$ for argument $t\in A$.
The remaining cases ($\{\ell,\ell_*\}, \{\ell,r_*\}$) does not apply as $\ell\not\in L'$.
{However, this is the case iff either $\{\ell,t\}\in R$ when $\ell=s$, or $\{\ell_*,t\}\in R$ when $\ell=s_*$ for some $s\in A$. This contradicts with the conflict-freeness of $L$, since either $t\in L$ or $t_*\in L$ is true.}
\end{proof}

\claimcorrcnfnaive*
\begin{proof}[Proof of Claim]
For any satisfying assignment $\theta$ for $\varphi$, one can construct a naive set $S$ in $F$ (and hence also in $F'$) of size $m$ due to Claim~\ref{claim:sat-naive}. 
Then, following Claim~\ref{claim:shadow}, we let $S$ as one $\sem$-extension and its shadow $S_*$ as another.
Since $|S|= m$ and $S\cap S_*=\emptyset$, we have that $S$ and $S_*$ are $2m$-diverse.

Conversely, suppose $F$ admits two $\sem$-extensions $S$ and $T$ of diversity $2m$.
From $S$ and $T$, we construct a satisfying assignment for $\varphi$.

First, we consider the sets $S_o\dfn S\setminus T$ and $T_o \dfn T\setminus S$.
We have that $|S_o\cup T_o|\geq 2m$ by our assumption.
As a result, either both sets have size at least $m$, or at last one of these sets has size larger than $m$.
The case when both sets are of smaller size is contrary to our assumption that their union has size $\geq 2m$.
We assume wlog that $|S_o|\geq m$ and let $S^I_o = \{s\mid s\in S_o\cap A\} \cup \{t \mid t_*\in S_o\}$. 
It is easy to observe that $S_o^I\subseteq A$.
From the conflict-freeness of $S$ and the fact that $|S^I_o|\geq m$, it follows that $S^I_o$ is a satisfying assignment for $\varphi$ (as we outline next).
First, it is easy to see that $S^I_o$ is a set of literals over $X$ by definition of $S^I_o$.

Now, we prove that $S^I_o$ gives a valid assignment, i.e., it does not contain two literals of opposite parity.
We prove that no two literals in $S^I_o$ are negations of each other.
Suppose $s,t\in S^I_o$. 
Then $S_o$ contains one of the pair $\{s,t\}, \{s,t_*\}, \{s_*,t\}$, or $\{s_*,t_*\}$ and this is also true for $S$ since $S_o\subseteq S$.
Therefore, $\{x,y\}\not\in R$ follows for each $x\in \{s,s_*\}$ and $y\in\{t,t_*\}$ since $S$ is conflict-free.
Consequently, we have that $s\neq \bar t$ for any literal $s$ and $t$ in $S^I_o$ by definition of $R$ and hence $S^I_o$ is valid assignment.
That $S^I_o$ satisfies $\varphi$ follows due to the size argument since  $|S^I_o|=m$ and hence contains one literal (argument) from each clause (triangle of attacks).
\end{proof}

\claimdom*
\begin{proof}[Proof of Claim]
Notice that here we require naive semantics due to the subset-maximality as it guarantee that for every argument $v\in A$ and a naive set $S\subseteq A$, either $S$ contains $v$, or $S\cup\{v\}$ is not conflict-free and hence there is some $u\in S$ with $\{u,v\}\in R$.
To prove the claim, we observe that a dominating set $S$ in $G$ of size at most $k$ gives a naive extension $S\cup\{a\}$ in $F$ of size at most $k+1$. 
Then, our second naive extension is the set $\{b\}$ and this proves the claim.
Conversely, $\{b\}$ is the only naive set in $F$ containing $b$. 
Hence, to achieve two $(k+2)$-diverse naive extensions for arguments $a$ and $b$, there must be a naive set $S$ in $F$ of size at most $k$.
Moreover, for any set $S$ with $a\in S$, $S$ is naive iff for each argument $x\not\in S$, there is some $y\in S$ such that $\{x,y\}\in R$ . As a result, $S$ is a dominating set of size at most $k$.
\end{proof}

\subsection{Omitted Proofs of Section~\ref{sec:adm}}

\thmdivarg*
\begin{proof}
Let $F=(A,R)$ be an instance and $a,b\in A$ be two arguments.
For membership, one can guess two witness sets $S_a, S_b\subseteq A$ such that (i) $S_a$ and $S_b$ are $\sem$-extension, (ii) $a\in S_a$, $b\in S_b$, and  (iii) $d(S_a,S_b)= k$. 
The verification requires polynomial time for each $\sem\in\{\adm,\comp,\stab\}$.
For $\abovekdiversearg_\sem$ and $\belowkdiversearg_\sem$, we replace equality in (iii) by $\geq k$ and $\leq k$, respectively, which can still be performed in polynomial time.

For hardness, we reduce SAT to the existence problem of $k$-diverse extensions containing two arguments.
Let $\varphi=\{C_1,\dots, C_m\}$ be a SAT instance where each $C_i$ is a clause for $i\leq m$ over variables $V=\{x_1,\dots,x_n\}$.
Moreover, let $F_\varphi=(A,R)$ be the AF generated from $\varphi$ via the standard translation for admissible semantics.
We construct an AF $F'=(A',R')$ where $A'=A\cup \{b\}$, for a fresh argument $b\not\in A$, $R'=R\cup \{(\varphi,b)\}\cup \{(b,x) \mid x\in A\}$ and let $k=n+2$.
The set $\{b\}$ is trivially a $\sem$-extension in $F'$ for each $\sem\in\{\adm,\comp,\stab\}$ and the following statements are equivalent.
\begin{itemize}
	\item The formula $\varphi$ is satisfiable.
	\item 
	$F$ has some $\sem$-extension $S$ with $\varphi\in S$.
	\item $F'$ has some $\sem$-extension $S$ with $\varphi\in S$ such that  $d(S,B)= k$ for $B=\{b\}$.
\end{itemize}
As a result, $\varphi$ is satisfiable 
if and only if arguments $\varphi$ and $b$ are contained in two $k$-diverse $\sem$-extensions of $F'$ for $\sem\in\{\adm,\comp,\stab\}$.

Regarding the claim for $\belowkdiversearg_\sem$ and $\abovekdiversearg_\sem$: we observe that the same reduction as illustrated above for $k$-diverse extensions applies as we are forced to take an extension containing $\varphi$ which is possible iff $\varphi$ is satisfiable.
\end{proof}

\thmdivext*
\begin{proof}
For membership, guess two witness sets as before.

For hardness, we reconsider the reduction used in the proof of Theorem~\ref{thm:div-arg}.
Observe that, while $\{b\}$ is already an extension, we are looking for another extension, which can only come from $F_\varphi$. Hence the exact same reduction applies.
Here, the trick is that $b$ attacks all the arguments in $F_\varphi$ and is in turn only attacked by the argument $\varphi$. This eliminates all $\sem$-extensions of $F_\varphi$ that do not contain $\varphi$.

Suppose $\varphi$ is satisfiable.
Then, there exists a $\sem$-extension $S$ containing $\varphi$.
This yields two $k$-diverse $\sem$-extensions, namely $S$ and $\{b\}$ in $F$.
Conversely, suppose $\varphi$ is not satisfiable.
Then, $F$ only admits one $\sem$-extension, which is $\{b\}$.
Hence, $F$ does not admits two $k$-diverse $\sem$-extensions for any $\sem\in \{\adm,\comp,\stab\}$.
The argument for $k$ follows by observing that any extension containing $\varphi$ in our AF must contain $\varphi$, and literal arguments corresponding to a satisfying assignment, thus has size $n{+}1$ where $n$ is the number of variables in $\varphi$.
\end{proof}

\subsection{Omitted Proofs of Section~\ref{sec:semi}}

\lemrange*
\begin{proof}
For brevity, we write $S_{\bar a} \dfn S\setminus \{a\} $ and $S_{\bar a b}\dfn S_{\bar a}\cup \{b\}$ for the set $S$ with $a\in S$.

\textbf{For (1)}: the first claim is trivial for the conflict-free semantics as we observe next.
Let $S\subseteq A$ be conflict-free such that $a\in S$, then $S_{\bar a}$ is also conflict-free, as $S_{\bar a}\subseteq S$. This implies that $S_{\bar a b}$ is also conflict-free since $b$ and $a$ have same incoming and outgoing attacks.
Analogous reasoning applies to the converse direction by reversing the roles of $a$ and $b$.

We now consider the case of admissibility.
Suppose $S$ is admissible and assume there are arguments $s\in S_{\bar a}$ and $c\in A\setminus S_{\bar a}$ such that $(c,s)\in R$. 
Since $S$ is admissible, we have two cases:
(i) either $s$ is defended by some argument in $S_{\bar a}$, and hence $S_{\bar a}$ is admissible, or  
(ii) $s$ was defended by $a$ and hence $(a,c)\in R$.
We consider the second case. 
By our assumption, $(b,c)\in R$ and thus the set $S_{\bar a b}\dfn S_{\bar a}\cup \{b\}$ defends $s$ against the attack by $c$.
Moreover, $S_{\bar a b}$ is also conflict-free by our first claim.
Since this holds for every argument $s\in S_{\bar a}$, the set $S_{\bar a b}$ is admissible in both cases.
For the reverse direction, similar reasoning applies by changing the roles of $a$ and $b$.

\textbf{For (2)}: let $x\in S^+_R$.
If $x\in S$ then we consider two cases: 
(i-a) if $x=a$, then $x\in S'^+_R$ since $b\in S'$ and $(b,a)\in R$.
(i-b) if  $x\neq a$, then $x\in S'$ and hence $x\in S'^+_R$ since $S\subseteq S'^+_R$ holds for any set $S\subseteq A$ of arguments.
Next, if $x\in S^+_R\setminus S$, we again consider the following two cases:
(ii-a) if $(a,x)\in R$, then $(b,x)\in R$ by our assumption and hence $x\in S'^+_R$ since $b\in S'$.
(ii-b) if $(y,x)\in R$ for some $y\neq a$, then $y\in S'$ and hence $x\in S'^+_R$ once again.
\end{proof}

\claimsemisigmap*
\begin{proof}[Proof of Claim]
The first equivalence follows by the definition of the standard reduction~\cite{DVORAK2010425}.
For the second, we observe that any $\sem$-extension in $F_\Phi$ is also a $\sem$-extension in $F'$ since the only difference between the two AFs lies in the argument $\bar\varphi_c$, which is attacked by (and attacks) both $\varphi$ and $\bar\varphi$ and it does not participate in any other attack.
Therefore, if $\Phi$ is true, then no $\sem$-extension in $F'$ contains $\bar\varphi$ and the same applies to its copy argument $\bar\varphi_c$ in $F'$ due to Lemma~\ref{lem:range}.
As a result, no extension in $F'$ contains either argument and hence $\bar\varphi$ and $\bar\varphi_c$ are not $2$-diverse in $F'$.
Conversely, suppose $\Phi$ is false.
Hence $\bar\varphi$ is contained in some $\sem$-extension $S$ of $F'$.
We let $S'\dfn S\setminus \{\bar\varphi\} $ and note that $S_c\dfn S' \cup \{\bar\varphi_c\}$ is a $\sem$-extension in $F'$ due to the 2nd item in Lemma~\ref{lem:range}.
Since $\bar\varphi$ does not attack any argument in $F_\Phi$ except $\varphi$, no argument in $S'$ requires $\bar\varphi$ to defend it. 
The claim that $S_c$ is semi-stable (resp., stage) follows from the fact that $S$ is semi-stable (stage) and that both sets have the exact same range, i.e., $S^+_{R'}= (S_c)^+_{R'}$.
Here, we once again apply 2nd item in Lemma~\ref{lem:range} for the range of two sets.
We conclude by observing that $d(S, S_c)=2$.
\end{proof}

\thmdivextsemi*
\begin{proof}
For membership, guess two witness sets as before.

For hardness, we reconsider the reduction from the proof of Theorem~\ref{thm:div-arg-semi} an update it as follows.
Let $F_\Phi=(A,R)$ be the AF constructed via the standard translation for the semi-stable (stage) semantics.
Then, we let $k= 2|X|+1$ and construct $F'=(A',R')$ as follows:
\begin{align*}
	A'\dfn & A\cup \{\bar\varphi_i \mid i\leq k\}, \\
	R'\dfn & R \cup \{(\varphi,\bar\varphi_i), (\bar\varphi_i,\varphi), (\bar\varphi,\bar\varphi_i),(\bar\varphi_i,\bar\varphi) \mid i\leq k\}.
\end{align*}
Our reduction follows the same intuition as before, but now the set of arguments $\{\bar\varphi_1,\dots,\bar\varphi_k\}$ serves as copies of $\bar\varphi$.
The correctness also follows by the same reasoning as in Claim~\ref{claim:semi-sigmap} as we prove in the following claim.

\begin{claim}
	$\Phi$ is false iff $F'$ admits two  $(k+1)$-diverse $\sem$-extensions for  $\sem\in\{\semi,\stag\}$.
\end{claim}

\begin{proof}[Proof of Claim]
	Suppose, $\Phi$ is true. 
	Then no $\sem$-extension in $F'$ contains $\bar\varphi$ and the same applies to its copy arguments $\bar\varphi_i$ in $F'$ due to Lemma~\ref{lem:range}.
	Observe that, any $\sem$-extension in $F'$ containing $\varphi$ arises due to a satisfying assignment for $\varphi$, and must be of size $|X|$. 
	Hence two $\sem$-extensions in $F'$ can only be maximally $2|X|$-diverse (when two satisfying assignments give completely diverse extensions, but still contain $\varphi$ as a common argument).
	Consequently, $F'$ does not admit two $(k+1)$-diverse extensions.
	
	Conversely, suppose $\Phi$ is false. 
	Hence, $\bar\varphi$ is contained in some $\sem$-extension $S$ of $F'$.
	We let $S'\dfn S\setminus \{\bar\varphi\} $ and observe that $S_c\dfn S' \cup \{\bar\varphi_i\mid i\leq k\}$ is a $\sem$-extension in $F'$ following the same reasoning as in Claim~\ref{claim:semi-sigmap}.
	We conclude by observing that $d(S, S_c)=k+1$.
	Moreover, it is easy to see that the reduction also applies to ${\geq k}$-diverse extensions.
\end{proof}
\end{proof}

\divargpref*
\begin{proof}
For membership, guess two witness sets as before.

For hardness, we utilize the reduction from~\cite[Reduction 3.7]{flap/DvorakD17} proving $\PiP$-hardness of skeptical acceptance with preferred semantics.
Given a QBF $\Phi = \forall Y \exists Z.\varphi$ where $\varphi\dfn\bigwedge_{i=1}^m C_i$ is a CNF-formula with clauses $C_i$ over variables $X = Y \cup Z$. 
The standard translation for preferred semantics yields the AF $F_\Phi=(A,R)$, where:
\begin{align*}
	A \dfn & \{\varphi,\bar \varphi\} \cup \{C_1 , \dots, C_m\} \cup X \cup \bar X, \\
	R \dfn &\{\,(C_i,\varphi)\mid i\leq m\,\}\cup \{\,(x,\bar x),(\bar x,x)\mid x\in X\,\}\\
	&\cup\{\,(x,C_i)\mid x\in C_i\,\}\cup\{\,(\bar x,C_i)\mid \bar x\in C_i\,\} \\
	& \cup \{(\varphi, \bar\varphi), (\bar\varphi, \bar\varphi)\,\}\cup \{\,(\bar\varphi, z), (\bar\varphi, \bar z)\mid z\in Z\,\}.
\end{align*}
It is known that every assignment $\theta_Y$ over $Y$ (seen as a set of literals) is an admissible set in $F_\Phi$.
Moreover, $\theta_Y$ yields a preferred extension in $F_\Phi$ iff there is a preferred extension in $F_\Phi$ (namely, $\theta_Y$) not containing the argument $\varphi$ iff the formula $\forall Y \exists Z.\varphi$ is false.

To complete our construction, we consider the AF $F'=(A',R')$ by taking a fresh argument $b\not\in A$, such that: 
\begin{align*}
	A'\dfn & A\cup \{b,\bar b\} \\
	R'\dfn & R \cup \{(\varphi,b)\} \cup \{(b,x) \mid x\in A\}\\
	& \cup\{(b,\bar b), (\bar b, b)\} \cup \{(y,b),(\bar y,b)\mid y\in Y\}
\end{align*}
Intuitively, the set $\{b\}$ is trivially a preferred extension in $F'$ whereas $\bar b$ alone is (although admissible) not a preferred set itself.
Moreover, if some assignment $\theta_Y$ over $Y$ is a preferred extension in $F_\Phi$, then $\theta_Y\cup\{\bar b\}$ is a preferred set in $F'$.
\future{\yasir{check later: maybe we dont need attacks from Y to b as we have $\bar b$.}}
We let $k= |Y|+2$ and prove the following claim.
\begin{claim}
	$\Phi$ is false iff argument $b$ and $\bar b$ are $k$-diverse in $F'$ under preferred semantics.
\end{claim}
\begin{proof}[Proof of Claim]
	If $\Phi$ is false then there is a preferred extension $\theta_Y$ in $F_\Phi$ which does not contain $\varphi$.
	Hence $F'$ admits a preferred extension $\theta_Y\cup\{\bar b\}$.
	Together with $B=\{b\}$, this yields two $k$-diverse extensions in $F'$.
	
	Conversely, if $\Phi$ is true, then every preferred extension in $F_\Phi$ contains $\varphi$.
	Each such preferred set $S$ also yields a preferred extension $S\cup\{\bar b\}$ in $F'$. 
	Here, the argument $b$ is not conflict-free with any $S$ as each $S$ contains $\varphi$.
	As a result, the only preferred set in $F'$ containing $b$ is the set $B=\{b\}$.
	Furthermore, each preferred set $S$ in $F'$ corresponding to a satisfying assignment must have size $|X|+2$ (one literal per each variable, and arguments $\{\varphi, \bar b\}$).
	Consequently, $F'$ can not have two $k$-diverse preferred extensions for $b$ and $\bar b$ as $|X|>|Y|$ (assuming that $Z\neq\emptyset$).
\end{proof}

We conclude by observing that the reduction also applies to $\belowkdiversearg_\pref$.
This holds since $\Phi$ is false iff $F'$ admits two preferred extensions containing arguments $b$ and $\bar b$, respectively, which are $\leq k$ diverse.
Here, we once again utilize the fact that even if $\varphi$ is satisfiable, any preferred set arising from a satisfying assignment must have size $|X|+2$.
\end{proof}

\thmdivextpref*
\begin{proof}
For membership, guess two witness sets as before.

For hardness, we use the similar idea as the construction in proof of Theorem~\ref{thm:div-arg-pref}.
Precisely, we do not use the argument $\bar b$.
A possible issue is that although $\{b\}$ is a preferred extension. 
Other preferred sets in $F'$ can either have size $|Y|$ (if $\Phi$ is false) or $|X|+1$ (otherwise).
Hence, $F'$ might already have two $|Y|$-diverse extensions coming from two satisfying assignments which share certain arguments and the equivalence between evaluation of $\Phi$ and achievable diversity breaks down.
Nevertheless, we adjust our reduction to model that both $Y$ and $X$ are of even cardinality.
This way, one extension containing only literals over $Y$, and another one  as $\{b\}$ would yield a diversity of $|Y|+1$, which is an odd number.
In contrast, any two preferred extensions corresponding to satisfying assignments of $\Phi$ would contain all the literals, and hence would be of even size. In such an event, whenever two preferred extensions in $F'$ share arguments, the diversity decreases by two and hence still remain an even number.


Formally, we consider the AF as in the proof of Theorem~\ref{thm:div-arg-pref} and construct an AF $F'=(A',R')$, where,
\begin{align*}
	A'\dfn & A\cup \{b\}\cup \{y_d \text{ if $|Y|$ is odd}\}\cup \{z_d \text{ if $|Z|$ is odd} \} \\
	R'\dfn & R \cup \{(\varphi,b)\} \cup \{(b,x) \mid x\in A\}\cup \{(b,y_d),(b,z_d)\}\\
	& \cup \{(y,b),(\bar y,b)\mid y\in Y\}
\end{align*}
Intuitively, $\{b\}$ is a preferred extension in $F'$ as before.
Whereas the remaining fresh arguments (if any) are only attacked by $b$, and are used to reason about the size of preferred extensions in $F'$.
For brevity, we assume henceforth that both sets $Y$ and $Z$ have odd sizes, and write $Y_d$ to mean $Y\cup\{y_d\}$  and an assignment $\theta_{Y_d}$ over $Y$ to indicate $\{y_d\} \cup \theta_Y$.

We let $k= |Y_d|+1$ (which is now an odd number) and prove the following claim.

\begin{claim}
	$\Phi$ is false iff $F'$ admits two $k$-diverse preferred extensions.
\end{claim}
\begin{proof}[Proof of Claim]
	If $\Phi$ is false then there is a preferred extension $\theta_Y$ in $F_\Phi$ and hence a preferred extension $\theta_{Y_d}$ in $F'$ which does not contain $\varphi$.
	This yields two $k$-diverse preferred extensions $\theta_{Y_d}$ and $\{b\}$ in $F'$.
	
	Conversely, if $\Phi$ is true, then every preferred extension in $F_\Phi$ contains $\varphi$.
	Each such preferred set $S$ also yields a preferred extension in $F'$ as the argument $\varphi\in S$ attacks the only new attacker $b$ in $F'$.
	Furthermore, each preferred set $S$ in $F'$ must have a size $|X_d|+1$ where $X_d= Y_d\cup Z_d$. 
	Since, $\varphi \in S$ for each $S$, any two preferred extensions in $F'$ must be at most $2|X_d|$-diverse. 
	Finally, if two preferred sets $S_1,S_2$ in $F'$ share arguments, the diversity decreases by $2$ as the common argument is dropped from both sets.
	Consequently, $F$ can not have two $k$-diverse preferred extensions since $k=|Y|+1$ is an odd integer.
\end{proof}
%
%
\end{proof}

\subsection{Omitted Proofs of Section~\ref{sec:max}}
For proofs in this section, we  exploit the reductions used in earlier sections for corresponding semantics and reduce from SAT-UNSAT problem for (quantified) Boolean formulas.
The idea is to (i) encode two SAT (or QBF) instances to one AF and (ii) define the diversity threshold $k$ in such a way that only the \emph{correct} formula is satisfied (true).
Precisely, the resulting AF admits two $k$-diverse extensions iff the first instance is satisfied (true) and $F$ does not admits any $k'$-diverse extension for $k'>k$ iff the second instance is not satisfiable (true).

In this section, we construct final AFs for our proofs by taking union of certain \emph{sub-AFs}. Moreover, we will also write ``$\sem$-extension'' to denote $\sem$-extensions of a sub-AF.
In each proof below, we establish an intuition on how the sub-AFs are combined and outline all the possible issues for any technical trick that we might require.
This intuition is followed by a formal definition of the constructed AF in each case.

\thmconfmaxkdiv*
\begin{proof}
For membership, observe that an AF $F=(A,R)$ has maximum $k$-diverse $\sem$-extensions iff 
(i) $F$ admits a pair $S_p, T_p$ of $\sem$-extensions with $d(S_p,T_p)=k$, and 
(ii) for each pair $S_n,T_n$ of $\sem$-extensions in $F$, $d(S_n,T_n)\leq k$.
Here, (i) is an $\NP$ and (ii) is a $\co\NP$ problem. Thus the problem $\maxkdiversecover_\sem$ is in $\DP$ for each $\sem\in \{\cnf,\nai\}$.

For hardness, we reduce from SAT-UNSAT, following the reduction idea used in the proof of Theorem~\ref{thm:conf-atleast-div-arg} as an intermediate step.
Let $\varphi \dfn \bigwedge_{i\leq m} C_i$ where $C_i= (p_{i1}\lor p_{i2}\lor p_{i3})$, and $\psi \dfn \bigwedge_{i\leq n} D_i$ where $D_i= (q_{i1}\lor q_{i2}\lor q_{i3})$ be two CNFs over variables $X$ and $Y$, respectively.
We construct the two sub-AFs $F_\varphi=(A_\varphi, R_\varphi)$ and $F_\psi=(A_\psi,R_\psi)$ via the translation for conflict-free semantics.
\begin{align*}
	A_\varphi\dfn & \{p_{ij} \mid 1\leq i\leq m, 1\leq j\leq 3 \},  \\
	R_\varphi \dfn & \{\{p_{i1}, p_{i2}\}, \{p_{i2}, p_{i3}\}, \{p_{i3}, p_{i1}\} \mid 1\leq i\leq m\} \\
	& \cup \{\{p,p'\} \mid  p=x, p'=\bar x \text{ for some } x\in X\} \\
	A_\psi \dfn & \{q_{ij} \mid 1\leq i\leq n, 1\leq j\leq 3 \}, \\
	R_\psi \dfn &  \{\{q_{i1}, q_{i2}\}, \{q_{i2}, q_{i3}\}, \{q_{i3}, q_{i1}\} \mid 1\leq i\leq n\} \\
	& \cup \{\{q,q'\} \mid  q=y, q'=\bar y \text{ for some } y\in Y\}. 
\end{align*}%
Due to Claim~\ref{claim:sat-naive}, we know that $\varphi$ (resp., $\varphi$) is satisfiable iff $F_\varphi$ ($F_\psi$) admits a conflict-free set of size $m$ ($n$).
However, when $\varphi$ is false, we don't know (at most) how many clauses of $\varphi$ can still be satisfied. Hence, we have no control over the size of conflict-free sets in the case when $\varphi$ is not satisfiable.
Nevertheless, we do know that any conflict-free set in $F_\varphi$ (resp., $F_\psi$) must have size at most $m-1$ ($n-1$).
Thus, we use this observation to distinguish positive and negative cases of SAT-UNSAT. 
Intuitively, we take copies of each sub-AFs, such that all the cases of which formula is true or false can be distinguished.
We construct an AF to encode this via the following steps.
\begin{itemize}
	\item We take $n$ copies of $F_\varphi$ (i.e., each argument and their attacks in $F_\varphi$) to obtain the AF $F'_\varphi$, which now has $3m\times n$ arguments.
	Observe that if $F_\varphi$ admits a conflict-free (resp., naive) set of size $s\leq m$, then $F'_\varphi$ admits a conflict-free (naive) set of size $s\times n$. This is achieved by taking each copy of the original conflict-free (naive) set in $F_\varphi$ and observing that there are no new attacks between (the arguments of) copies for each sub-AF.
	Moreover, any $\sem$-set of size  $s\times n$ in $F'_\varphi$ also yields a $\sem$-set of size at least $s$ in $F_\varphi$ for $\sem\in\{\cnf, \nai\}$.
	This holds as one can always consider a copy of $F_\varphi$ from which maximal arguments are taken in the given $\sem$-extension, which yields an extension of corresponding size $\geq s$ in $F_\varphi$.
	\item We take altogether two copies of $F_\psi$ to obtain $F'_\psi$.
	Likewise, now $F_\psi$ admits a naive/conflict-free set of size $s\leq n$ iff  $F'_\psi$ admits a naive/cf set of size $2s$.
	\item Next, we consider a set $B=\{b_i \mid i\leq 2n-1\}$ of fresh auxiliary arguments, without any attacks between them, which we will denote by the sub-AF $B$.
\end{itemize}
Finally, we add attacks so that each argument in one sub-AF 
is in conflict with each argument in both other sub-AFs. 
This allows us to encode that any naive set in either sub-AF $B,F'_\varphi, F'_\psi$ is in fact a naive set in $F$.

Intuitively, a satisfying assignment for $\varphi$ yields a $\sem$-extension in $F'_\varphi$ (and also in $F$, by definition of $R$) of size $m\times n$.
Now, either $\psi$ is also satisfiable and hence $F'_\psi$ and $F$ admits a $\sem$-extension of size $2n$, or $\psi$ is not satisfiable and hence the other $\sem$-extension that can yield maximum diversity is the set $B$.
There remains a final issue. 
If $\varphi$ admits two satisfying assignments that do not share any literal, then $F_\varphi'$ admits two $\sem$-extensions which are $2(m\times n)$-diverse as they do not share any argument.
However, we do not know in advance if this case might occur. Nevertheless, we set our $K$ in such way that it enforces to select only one $\sem$-extension from $F_\varphi'$.
To this aim, we further consider a set $B_\varphi\dfn \{a_j \mid j\leq  m\times n\}$ of auxiliary arguments in $F_\varphi'$ so that any $\sem$-extension in $F_\varphi'$ contains these arguments as well.
This allows us to model that when $\varphi$ has two satisfying assignments, then their corresponding $\sem$-extensions in $F_\varphi'$ both contain $B_\varphi$ and hence does not allow to achieve a diversity value larger than $K$.

For a set $X$ of arguments, we denote by $X^i=\{x^i\mid x\in X\}$  a fresh set of arguments, as a copy of $X$.
Similarly, for a set $R$ of attacks over $X$, $R^i=\{(x^i,y^i)\mid (x,y)\in R\}$ then defines the attacks copy of $R$ for $X^i$.
Now we are ready to define our final AF.

Formally, we let $A_\varphi'\dfn  \bigcup_{c\leq n} A^c_\varphi\cup B_\varphi$, $A_\psi'\dfn A_\psi \cup A_\psi^c$, and similarly denote $R'_\varphi \dfn  \bigcup_{c\leq n} R^c_\varphi$, and  $F_\psi'\dfn R_\psi \cup R^c_\psi$.
As our final AF, we consider $F=(A,R)$, where 
\begin{align*}
	A \dfn  A'_\varphi & \cup A'_\psi  \cup B, \\
	R \dfn R'_\varphi & \cup R'_\psi \\
	& \cup \{\{s,b\}\mid s\in A'_\varphi, b\in B\} \\
	& \cup \{\{t,b\}\mid t\in A'_\psi, b\in B\} \\
	&\cup \{\{s,t\}\mid s\in A'_\varphi, t\in A'_\psi\}. 
\end{align*}

To complete our reduction, we set $K= 2(m\times n) + 2n-1$ and prove the following claim.
\begin{restatable}[$\star$]{claim}{claimqbf}\label{claim:qbf}
	$(\varphi,\psi)$ is a positive instance of SAT-UNSAT iff $F$ has two $K$-diverse $\sem$-extensions but no $K'$-diverse $\sem$-extensions for any $K'>K$.
\end{restatable} 
\begin{proof}[Proof of Claim]
	Suppose $\varphi$ is satisfiable. 
	Then $F_\varphi$ admits a $\sem$-extension $S_\varphi$ of size $m$ as in the proof of Claim~\ref{claim:sat-naive}.
	Now, let $S'_\varphi = S_\varphi\cup \{a^c \mid a\in S_\varphi\}$.
	Then, $S_\varphi'$ is conflict-free and has size $m\times n$.
	We let $S=S'_\varphi\cup B_\varphi$ and prove that $S$ is in fact naive in $F$.
	First, adding any other argument from $A_\varphi$ to the set $S'_\varphi$ violates conflict-freeness of $S_\varphi$ due to our Claim~\ref{claim:sat-naive}.
	The same applies to all the copies of arguments in $A_\varphi$, and hence to the set $A_\varphi'$.
	Second, every argument in $S_\varphi'$ attacks all arguments in $B$ and $F_\psi'$ by definition.
	As a result, the $S$ is indeed naive in $F$ and hence $F$ admits a $\sem$-extension $S$ of size $2(m\times n)$.
	Next, we let $B$ as our second conflict-free (naive) set in $F$ of size $2n-1$.
	This yields $S$ and $B$ as two $K$-diverse $\sem$-extension in $F$.
	
	We next prove that $F$ does not admit any $K'$-diverse $\sem$ extensions for $K'>K$.
	Due to $m$ ``triangles'' of attacks in $F_\varphi$, the largest $\sem$-extension in $F'_\varphi$ can be of size $m\times n$.
	Whereas, the only possibility to achieve a larger diversity is when $F_\psi'$ admits an extension larger than $2n-1$ which is only possible if $F_\psi$ admits a $\sem$-extension of size $n$.
	Since $\psi$ is not satisfiable, any $\sem$-extension in $F_\psi$ has size at most $n-1$ and hence $F'_\psi$ can only have $\sem$-extensions of size at most $2(n-1)$.
	Finally, if $\varphi$ admits two satisfying assignments, their corresponding $\sem$-extensions $S_1,S_2$ in $F$ have size $2(m\times n)$ each due to our claim above but they also share the set $B$. Hence $d(S_1,S_2)=m\times n$ in this case and two $\sem$-extension from $F_\varphi'$ can not incur a diversity of $K$.
	Thus, $F$ does not admit $K'$-diverse $\sem$-extensions for $K'>K$.
	
	Conversely, suppose $(\varphi,\psi)$ is a negative instance of SAT-UNSAT. We consider the following two cases.
	
	\textbf{Case-1: both $\varphi$ and $\psi$ are satisfiable.} As before, this implies that $F_\varphi'$ admits a $\sem$-extension of size $2(m\times n)$.
	Moreover, $F_\psi$ also has a $\sem$-extension $S_\psi$ of size $n$. 
	We let $S_\psi'=S_\psi \cup \{t^c \mid t\in S_\psi\}$ and observe that $S_\psi'$ is a $\sem$-extension in $F$ of size $2n$.
	This leads to a contradiction since $F$ admits two $(2(m\times n) + 2n)$-diverse $\sem$-extensions. 
	
	\textbf{Case-2: $\varphi$ is unsatisfiable.}
	In this case, the largest $\sem$-extension in the sub-AF $F_\varphi$ has size $(m-1)$ due to the proof of Claim~\ref{claim:sat-naive}.
	Thus, the largest $\sem$-extension in $F'_\varphi$ has size $(m-1)\times n + m\times n$. Here the second factor $m\times n$ is due to the set $B_\varphi$ which can still be added to any extension of the sub-AF $F'_\varphi$.
	Consequently, irrespective of whether $\psi$ is satisfiable or not, $F$ can not admit two $K$-diverse extension.
	This holds, since the largest $\sem$-extension in the remaining sub-AFs have sizes $2n-1$ due to $B$ when $\psi$ is unsatisfiable, or $2n$ due to any satisfying assignment of $\psi$.
	Thus $F$ admits maximally $K'$-diverse extensions where $K'= (m-1)\times n + m\times n+ 2n =2(m\times n) +n$ and therefore $K'<K$.
	
	This proves our claim since in either case, either $F$ does not admits two $K$-diverse $\sem$-extensions, or it also admits $K'$-diverse extensions for $K'>K$.
\end{proof}
\end{proof}

\thmmaxext*
\begin{proof}
We only outline the main idea for changes on top of the proof of Theorem~\ref{thm:conf-max-k-div}.
Here, we need to add two \emph{signal} arguments for which we can ask the $\sem$-existence of two $K$-diverse extensions in $F$.
First, we add a fresh argument $a_1$ to $F_\varphi'$  which does not participate in any attack in the sub-AF $F_\varphi'$.
Whereas, we add attacks so that $a_1$ is in conflict with all arguments in $B$ and $F_\psi'$.
Moreover, we also add another auxiliary argument $a_2$ such that it only attacks arguments in $F_\varphi'$.
Intuitively, $a_2$ can be taken in any $\sem$-extension for either the sub-AF $B$ or for $F'_\psi$, but not for $F_\varphi'$.

Formally, (as in the proof of Theorem~\ref{thm:conf-max-k-div}), we let $A_\varphi'\dfn  \bigcup_{c\leq n} A^c_\varphi\cup B_\varphi \cup \{a_1\}$, $A_\psi'\dfn A_\psi \cup A_\psi^c$, and similarly denote $R'_\varphi \dfn  \bigcup_{c\leq n} R^c_\varphi$, and  $F_\psi'\dfn R_\psi \cup R^c_\psi$.
As our final AF, we consider $F=(A,R)$, where 
\begin{align*}
	A \dfn  A'_\varphi & \cup A'_\psi  \cup B \cup \{a_2\}, \\
	R \dfn R'_\varphi & \cup R'_\psi \\
	& \cup \{\{s,b\}\mid s\in A'_\varphi, b\in B\cup\{a_2\}\} \\
	& \cup \{\{t,b\}\mid t\in A'_\psi, b\in B\} \\
	&\cup \{\{s,t\}\mid s\in A'_\varphi, t\in A'_\psi\cup\{a_2\}\}.
\end{align*}

Now, compared to the proof of Theorem~\ref{thm:conf-max-k-div}, the size of each extension in $F_\varphi'$ increases by one due to $a_1$.
Similarly, the size of each $\sem$-extension in $B$ or $F_\psi'$ also increases by one, due to $a_2$.
Then, we set $K= 2(m\times n) +2n+1$ and prove the following claim.

\begin{claim}
	$(\varphi,\psi)$ is a positive instance of SAT-UNSAT iff the arguments $a_1$ and $a_2$ are maximally $K$-diverse in $F$ under semantics $\sem\in\{\cnf,\nai\}$.
\end{claim} 

\begin{proof}[Proof of Claim]
	Suppose $\varphi$ is satisfiable. 
	Then, following proof of Claim~\ref{claim:qbf}, $F'_\varphi$ admits a $\sem$-extension $S$ of size $2(m\times n)+1$.
	The only difference lies in the additional auxiliary argument $a_1\in S$, which still renders $S$ conflict-free. Moreover due to the new attacks to the fresh argument $a_2$, $S$ is also naive in $F$.
	As a result, $F$ admits a $\sem$-extension $S$ of size $2(m\times n)+1$ with $a_1\in S$.
	Next, we let $B\cup \{a_2\}$ as our second conflict-free (naive) set in $F$ of size $2n$.
	This proves that arguments $a_1$ and $a_2$ are $K$-diverse in $F$.
	
	The claim that $F$ does not admit any $K'$-diverse $\sem$ extensions for $K'>K$ follows the exact same reasoning as in the proof of Claim~\ref{claim:qbf}.
	As before, the only possibility to achieve a larger diversity is when $F_\psi'$ admits an extension larger than $2n$ which is only possible if $F_\psi$ admits a $\sem$-extension of size $n$.
	Here, although there is one additional argument (namely, $a_2$) $F_\psi'$ must still admits a $\sem$-extension of size $2n$ to yield a larger extension than $B\cup \{a_2\}$.
	Whereas, if $\psi$ is not satisfiable, $F'_\psi$ can only have $\sem$-extensions of size at most $2(n-1)$ and hence adding $a_2$ does not help reaching the size $2n$.
	Finally, for any $\sem$-extensions $S_1,S_2$ in $F_\varphi'$ due to two satisfying assignments, although $a_1\in S_i$, we have $a_2\not\in S_i$ for any $i=1,2$.
	Thus, $F$ does not admit $K'$-diverse $\sem$-extensions for $K'>K$ covering $a_1$ and $a_2$.
	
	Conversely, suppose $(\varphi,\psi)$ is a negative instance of SAT-UNSAT. We consider the following two cases.
	
	\textbf{Case-1: both $\varphi$ and $\psi$ are satisfiable.} This implies that $F_\varphi'$ admits a $\sem$-extension $S_1$ of size $2(m\times n)+1$ containing $a_1$.
	Moreover, $F'_\psi$ also has a $\sem$-extension $S_\psi$ of size $2n$ and yields a $\sem$-extension $S_2= S_\psi\cup\{a_2\}$ in $F$.
	This leads to a contradiction since $F$ admits two $(2(m\times n) + 2n+2)$-diverse $\sem$-extensions containing $a_1$ and $a_2$, respectively. 
	
	\textbf{Case-2: $\varphi$ is unsatisfiable.}
	This case follows analogous reasoning as in the proof of Claim~\ref{claim:qbf}.
	The largest $\sem$-extension in $F'_\varphi$ has size $(m-1)\times n + m\times n +1$.
	Consequently, irrespective of whether $\psi$ is satisfiable or not, $F$ can not admit two $K$-diverse extension for arguments $a_1$ and $a_2$.
	
	This completes the proof to our claim since in each case, either $F$ does not admits two $K$-diverse $\sem$-extensions for $a_1$ and $a_2$, or it also admits $K'$-diverse extensions for $K'>K$.
\end{proof}
\end{proof}

\lemunion*
\begin{proof}
We consider the properties in individual semantics separately.

\textbf{1. Conflict-freeness.} Suppose $S_i$ is a conflict-free set in $F_i$ for $i=1,2$.
Then $S_1\cup S_2$ is conflict-free in $F$ as there is no $x\in S_1, y\in S_2$ such that $(x,y)\in R$ (by definition of $R$).
Conversely, let $S$ be conflict-free in $F$ and let $S_i= S\cap A_i$ for $i=1,2$.
Then, we use the fact that a subset of a conflict-free set is also conflict-free to establish that $S_i$ is conflict-free in $A_i$ for $i=1,2$.

\textbf{2. Admissibility.} Suppose $S_i$ be admissible in $F_i$ for $i=1,2$.
Then, being conflict-free $S=S_1\cup S_2$ is also conflict-free in $F$ due to (1).
Now let $s\in S$ and $a\in A\setminus S$ such that $(a,s)\in R$.
By definition of $R$, we have that $a,s\in A_i$ and $(a,s)\in R_i$ for the same $i\in\{1,2\}$.
Moreover, since $S_i$ is admissible for both $i=1,2$, there is some $b\in S_i$ such that $(b,a)\in R_i$ for the same $i$.
Hence the argument $s\in S$ is defended against the attack by $a\in A$ by the argument $b\in S$ as $S= S_1\cup S_2$. Therefore, $S$ is admissible.
Conversely, let $S$ be an admissible set in $F$ and consider $S_i= S\cap A_i$ for $i=1,2$.
We prove that both sets $S_i$'s are admissible in $F_i$.
Let $s\in S_1$ and $a\in A_1\setminus S_1$ be such that $(a,s)\in R_1$.
Then, due to $s\in S$, there must be some $b\in S$ with $(b,a)\in R$ since $S$ is admissible in $F$.
Since $R$ does not add any attacks between arguments of $A_1$ and $A_2$, this implies that $b\in S_1$ and $(b,a)\in R_1$ is in fact true.
Hence $S_1$ is admissible in $F_1$ and proves our claim. Analogous reasoning applies to $S_2$ and completes our claim.

\textbf{3. Subset-maximality.} We consider the case of naive and preferred semantics.
We only prove the claim for naive semantics, but the idea applies to preferred semantics after replacing ``conflict-free'' extensions by ``admissible'' ones.
Suppose $S_i$ is naive in $F_i$.
Then due to (1), $S_i$ is also conflict-free in $F_i$ and hence $S=S_1\cup S_2$ is conflict-free in $F$.
To prove that $S$ is naive, let $s\in A\setminus S$.
Let $s\in A_1$, as the case of $s\in A_2$ follows analogously.
Suppose to the contrary that $S'=S\cup \{s\}$ is conflict-free.
Then, being its subset, $S_1'=S'\cap A_1$ is also conflict-free in $F_1$.
But this is a contradiction since $S_1'= S_1\cup\{s\}$ and $S_1$ is naive in $F_1$. Hence $S$ is naive in $F$.
Conversely, let $S$ be a naive set in $F$ and let $S_i= S\cap A_1$ for $i=1,2$.
Then, once again, $S_i$ is conflict-free in $F_i$.
To prove that $S_i$ is in fact naive (for both $i=1,2$), we let $s\in A_i\setminus S_i$.
Suppose, $S_i\cup\{s\}$ is conflict-free in $F_i$. But this implies that $S\cup\{s\}$ is also conflict-free in $F$, which is a contradiction since $S$ was taken to be a naive set (and hence subset-maximal conflict-free) in $F$.

\textbf{4. Range.} To prove the remaining claims, we prove that for $S= S_1\cup S_2$, $S^+_R= (S_1)^+_{R_1}\cup (S_2)^+_{R_2}$ where $S_i= S\cap A_i$ for $i=1,2$.
Let $x\in S^+_R$. If $x\in S$ the claim follows trivially as either $x\in S_1$ or $x\in S_2$.
Now, suppose there exists $y\in S$ such that $(y,x)\in R$.
Then, again, either $x,y\in S_1$ or $x,y\in S_2$ since only attacks in $R$ are among arguments from the same set.
Moreover, since $F$ does not introduce any new attacks, we have that $(y,x)\in R_i$ for the same $i$ such that $y,x\in S_i$.
Therefore, we have that $x\in (S_1)^+_{R_1}\cup (S_2)^+_{R_2}$ and proves our claim in this direction.
Now, suppose $x\in (S_1)^+_{R_1}\cup (S_2)^+_{R_2}$.
The claim when $x\in S_1\cup S_2$ follows the same reasoning as above.
Moreover, when $x\in (S_i)^+_{R_i}$ for any $i=1,2$, we once again use the definition to infer that there is some $y\in S_i$ with $(y,x)\in R_i$.
As a result, $(y,x)\in R $ and hence $x\in S^+_R$.

\textbf{Remaining Semantics} For stable semantics: we observe that $S^+_R= A$ iff $(S_i)^+_{R_i} = A_i$ for any set $S\subseteq A$ with $S_i = S\cap A_i$ and $i=1,2$.
Similarly, for semi-stable and stage semantics: $S^+_R$ is subset-maximal in $F$ iff $(S_i)^+_{R_i}$ are both subset-maximal in $F_i$ for any set $S\subseteq A$ with $S_i = S\cap A_i$ and $i=1,2$.
The claim for both semantics also uses the reasoning for maximality as established in (3).
This completes all our cases and hence the proof to our lemma.
\end{proof}

\maxext*
\begin{proof}
For membership, observe that an AF $F=(A,R)$ has maximum $k$-diverse $\sem$-extensions covering arguments $a$ and $b$ iff 
(i) $F$ admits a pair $S_p, T_p$ of $\sem$-extensions with $d(S_p,T_p)=k$ with $a\in S_p, b\in T_p$, and 
(ii) for each pair $S_n,T_n$ of $\sem$-extensions in $F$, either $a\not\in S_n, b\not\in T_n$ or $d(S_n,T_n)\leq k$.
Here, (i) is an $\NP$ and (ii) is a $\co\NP$ problem. Thus the problem $\maxkdiversearg_\sem$ is in $\DP$ for each $\sem\in \{\adm,\comp,\stab\}$.

For hardness, we reduce from the SAT-UNSAT problem by using the standard translation for admissible semantics as an intermediate step.
Let $\varphi=\{C_1,\dots,C_m\}$ over variables $V_\varphi$ and $\psi=\{D_1,\dots,D_n\}$ over variables $V_\psi$ be two CNFs.
Furthermore, let $|V_\varphi|=k_1$,  $|V_\psi|=k_2$.
We construct two AFs $F_\varphi=(A_\varphi,R_\varphi)$ and $F_\psi=(A_\psi,R_\psi)$ for $\varphi$ and $\psi$ via the standard translation and take their disjoint union to give the AF $F'= (A',R')$ illustrated as follows.
\begin{align*}
	A'\dfn & A_\varphi \cup A_\psi, \\
	R' \dfn & R_\varphi\cup R_\psi.
\end{align*}
In the following we will refer to $F_\varphi$ and $F_\psi$ as sub-AFs.

First, as in the proof of Theorem~\ref{thm:div-arg}, we add an auxiliary argument $b_i$ for $i\in\{\varphi,\psi\}$ in each sub-AF $F_i$, which attacks all the arguments in $F_i$, and is in-turn only attacked by arguments $\varphi$ and $\psi$, respectively.
We still use this trick to invalidate any $\sem$-extension in the sub-AF $F_\varphi$ or $F_\psi$ that does not contain $\varphi$.
We aim at applying Lemma~\ref{lem:union} to consider extensions in sub-AFs so that their union results in the desired $\sem$-extension we are looking for.

Next, we just need to distinguish the positive case ($\varphi$ is satisfiable and $\psi$ is not satisfiable) from all the negative ones.
Intuitively, we want to set $K$ in such a way that $\varphi$ is satisfiable and $\psi$ is not satisfiable iff argument $\varphi$ and $b_\psi$ are maximally $K$-diverse in our final $F$.
Here, one could consider sub-extensions $E_\theta= \theta\cup\{\varphi\}$ and $\{b_\varphi\}$ in $F_\varphi$, together with the only $\sem$-extension $\{b_\psi\}$ of $F_\psi$.
This yields $\sem$-extensions $E_\theta\cup\{b_\psi\}$ and $\{b_\varphi,b_\psi\}$.
Thus, one may be tempted to set $K= k_1+2$ since the two extensions only share the argument $b_\psi$ and are this $(k_1+2)$-diverse.
Nevertheless, one can also construct other extensions for $\varphi$ and $b_\psi$ which are more diverse than $K$. 
Suppose when $\varphi$ has two satisfying assignments $\theta_1,\theta_2$, such that $\theta_1\cap\theta_2=\emptyset$.
Then, the sub-AF admits two $\sem$-extensions $E_i=\theta_i\cup\{\varphi\}$ for $i=1,2$.
As a result, $F$ admits $\sem$-extensions $E_1'= E_1\cup \{b_\psi\}$ and $E_2'= E_2\cup\{b_\psi\}$, 
which contain arguments $\varphi$ and $b_\psi$, and are $2k_1$-diverse in $F$. Hence $K$ is not maximal for $\varphi$ and $b_\psi$.
This issue can be resolved by a small trick.
We replace the auxiliary argument $b_\varphi$ in $F_\varphi$ by a set $B$ of arguments which is large enough to force us to take $B$ in a witness extension of a certain size.
Precisely, we let $N=\max\{|A_\varphi|,|A_\psi|\}$ and consider the set $B=\{b_i\mid i\leq M\}$ for $M=2N$.

Formally, as our final AF, we let $F=(A,R)$, where
\begin{align*}
	A =  A' & \cup B\cup \{b_\psi\}, \\
	R=  R' & \cup  \{(\varphi,b)\mid b\in B\} \cup\{(b,s) \mid b\in B, s\in A_\varphi\}\\
	& \cup  \{(\psi,b_\psi)\} \cup\{(b_\psi,t) \mid t\in A_\psi\}\\
\end{align*}%
We denote by $F'_\varphi$ and $F'_\psi$ the sub-AFs involving additional arguments $B$ and $\{b_\psi\}$ and their corresponding attacks, respectively.
It can be seen that $F'_\varphi$ and $F_\psi$ are still disjoint and that $F= F'_\psi\cup F'_\psi$. Thus Lemma~\ref{lem:union} is still applicable.
To complete our reduction, we let $K=k_1+M+1$ and prove the following claim.
\begin{claim}\label{claim:adm-DP-arg}
	$(\varphi,\psi)$ is a positive instance of SAT-UNSAT iff the arguments $b_\psi$ and $\varphi$ are maximally $K$-diverse in $F$ under semantics $\sem\in\{\adm,\comp,\stab\}$.               \end{claim}
\begin{proof}[Proof of Claim]
	Suppose, $\varphi$ is satisfiable and $\psi$ is not.
	Then, we construct two $K$-diverse $\sem$-extensions $S$ and $T$ in $F$ such that $\varphi\in S$ and $b_\psi\in T$, as follows.
	First, observe that $B$  is trivially a $\sem$-extension in $F'_\varphi$ as every $b\in B$ attacks all other arguments in $A_\varphi\setminus B$.
	The same applies to the set $\{b_\psi\}$ in the sub-AF $F'_\psi$.
	
	Let $E_\theta= \theta\cup\{\varphi\}$ for a satisfying assignment $\theta$ of $\varphi$ and observe that $E$ is a valid $\sem$-extension in $F'_\varphi$ by correctness of the standard translation for admissible semantics.
	Since $\psi$ is not satisfiable, $\{b_\psi\}$ is the only (non-empty) $\sem$-extension in $F'_\psi$.
	Now, consider sets $S=E_\theta\cup \{b_\psi\}$ and $T= B\cup \{b_\psi\}$, which are both $\sem$-extensions in $F$ due to Lemma~\ref{lem:union}.
	We conclude by observing that $d(S,T)=(k_1+1)+2N$ since two $S$ and $T$ only share one argument $b_\psi$.
	
	Next, to prove that there are no $K'$-diverse $\sem$-extensions for the arguments $\varphi$ and $b_\psi$ for $K'>K$, we note that no other extension in $F$ can yield a larger diversity.
	Indeed, the only other possibility is to replace $B$ by an extension $E_\alpha$ for a different satisfying assignment $\alpha$ of $\varphi$. 
	Nevertheless, it holds that $|E_\varphi|\leq k_1+1$ and thus the resulting extension $E_\alpha\cup\{b_\psi\}$ incurs a smaller diversity value than the set $B\cup\{b_\psi\}$ when taken together with the extension $S$ .
	This proves our claim in this direction.
	
	Conversely, suppose $(\varphi,\psi)$ is a negative instance of SAT-UNSAT. We consider the following two cases.
	
	\textbf{Case-1: both $\varphi$ and $\psi$ are satisfiable.} Then, one can construct two extensions in the sub-AFs by letting a satisfying assignments of each formula.
	Precisely, this results in a $\sem$-extension $E= E_\varphi \cup E_\psi $ in $F$, where $E_* = \theta_* \cup \{*\}$ and $\theta_*$ is a satisfying assignment for $*\in\{\varphi,\psi\}$. 
	Moreover, $|E|= k_1+k_2+2$.
	Now, together with $B'= B\cup\{b_\psi\}$, we have that $d(E, B')= (k_1+k_2+2) + (2N+1) = K+k_2+1$ where $K=2N+k_1+1$. Consequently, $F$ admits two $K'$-diverse $\sem$-extension for $K'>K$ and hence arguments $\varphi$ and $b_\psi$ are not maximally $K$-diverse in $F$, which is a contradiction.
	
	\textbf{Case-2: $\varphi$ is unsatisfiable.}
	In this case, it does not matter whether $\psi$ is satisfiable or not, as no $\sem$-extension in the sub-AF $F_\varphi$ contains $\varphi$.
	Consequently, arguments $\varphi$ and $b_\psi$ can not be $k$-diverse for any $k\in\mathbb N$.
	
	This completes our claim as in each case, either arguments $\varphi$ and $b_\psi$ are not $K$-diverse, or they are also $K'$-diverse for some $K'>K$.
\end{proof}
\end{proof}

\thmmaxkdiv*
\begin{proof}[Proof of Statement 1]
Membership follows as in the poof of Theorem~\ref{thm:max-k-div-adm}.

For hardness, we reconsider the reduction from SAT-UNSAT in the proof of Theorem~\ref{thm:max-k-div-adm}.
%
%
%
We observe that the maximum diversity in our final AF can be achieved only when the ``correct'' formula is true.
However, in the reverse direction for the proof of Claim~\ref{claim:adm-DP-arg}, the two trivial cases (when $\varphi$ is not satisfiable) can still be problematic.
When $\psi$ is true but $\varphi$ is false, then the only stable set in $F_\varphi$ is $B$.
Thus any extension in $F$ must use the sub-extension $B$ in $F_\psi$. 
As a result, only sub-extensions in $F'_\psi$ contribute towards the diversity value, and a maximum diversity of $K$ can not be achieved as $K> 2|A_\psi|$ by our construction. Consequently, the claim still follows.
Whereas, when we consider admissible and complete semantics, any subset of $B$ now yields valid extensions in $F_\psi$ and thus one might be able to construct two $K$-diverse extensions in $F$ by taking some \textit{suitable} combinations of arguments in $B$ together with extensions in the sub-AF $F'_\psi$.
Thus, it is not immediate how to argue that $K$ is no longer the maximal achievable diversity for these cases.

Nevertheless, we can apply a technical trick to overcome this issue. Intuitively, we want to avoid the case that one can construct $K$-diverse $\sem$-extensions from $F$ by splitting arguments in $B$.
This can be achieved by adding 
(i) further auxiliary arguments of the same size as $B$ and 
(ii) a ``long cycle'' over arguments in $B$ in such a way that either all arguments from $B$ are taken in an admissible extension of $F'_\varphi$, or none is taken.
Finally, to simplify our reasoning, we avoid any extensions over these newly added auxiliary arguments by adding self-attacks over them so that they do not belong to any $\sem$-extension.
This is in fact the same idea as considered in Example~\ref{ex:levels-adm} to force/avoid certain admissible sets.

Formally, we consider two sub-AFs $F_\varphi=(A_\varphi,R_\varphi)$ and $F_\psi=(A_\psi,R_\psi)$ for $\varphi$ and $\psi$ via the standard translation as in the proof of Theorem~\ref{thm:max-k-div-adm}.
Moreover, we consider a set $C=\{c_i\mid i\leq M\}$ of fresh arguments and construct the AF $F=(A,R)$, where, 
\begin{align*}
	A =  A_\varphi\cup A_\psi & \cup B \cup C \cup \{b_\psi\}, \\
	R=  R_\varphi\cup R_\psi& \cup  \{(\varphi,b)\mid b\in B\} \cup\{(b,s) \mid b\in B, s\in A_\varphi\}\\
	& \cup  \{(\psi,b_\psi)\} \cup\{(b_\psi,t) \mid t\in A_\psi\}\\
	& \\
	& \cup \{(b_i,c_{i}),(c_i,b_{i+1}) \mid i\leq M-1 \}\\
	& \cup \{(b_M,c_M), (c_M,b_1), (c_1,c_1)\}\\
	& \cup \{(\varphi,c)\mid c\in C\}.
\end{align*}
We denote by $F'_\varphi$ and $F'_\psi$ the sub-AFs involving additional arguments $B\cup C$ and $\{b_\psi\}$ and their corresponding attacks, respectively.
The two sub-AFs are still disjoint from each other and do not contain any attack from one to the other.
Moreover, no argument in the set $C$ is contained in any admissible set of $F'_\varphi$ since it requires all the arguments in $C$ to defend them against attacks from $B$, whereas, $C$ is then not conflict-free due to the attack $(c_1,c_1)$.
Finally, the attacks in the last line are added to handle the stable semantics as any extension in $F'_\varphi$ containing $\varphi$ must yield a stable extension, and thus must attack all arguments in $C$.
As before, we let $K=k_1+M+1$ and prove the following claim.

\begin{claim}\label{claim:adm-DP}
	$(\varphi,\psi)$ is a positive instance of SAT-UNSAT iff $F$ has two $K$-diverse $\sem$-extensions but no $K'$-diverse $\sem$-extensions for any $K'>K$.
\end{claim}
\begin{proof}[Proof of Claim]
	Suppose, $\varphi$ is satisfiable and $\psi$ is not.
	Then, we construct two $K$-diverse $\sem$-extensions in $F$.
	Namely, we let $E_1= \theta\cup\{\varphi, b_\varphi\}$ for a satisfying assignment $\theta$ of $\varphi$ and set $S=E_1\cup \{b_\psi\}$.
	For our second extension, we let $B'= B\cup\{b_\psi\}$.
	The extensions $S$ and $B'$ can only be made larger if either some arguments from $C$, or from $F_\psi$ are added, which is not possible as $C$ is not conflict-free and $\psi$ is not satisfiable. 
	Hence we have that $d(E, B')=K$ .
	
	The claim that $F$ does not admit $K'$-diverse $\sem$ extensions for any $K'>K$ follows the exact same reasoning as in the proof of Claim~\ref{claim:adm-DP-arg}.
	
	Conversely, suppose $(\varphi,\psi)$ is a negative instance of SAT-UNSAT. We consider three cases.
	
	\textbf{Case-1: both $\varphi$ and $\psi$ are satisfiable.} This case follows the same reasoning as in the proof of Claim~\ref{claim:adm-DP-arg}.
	Thus, we consider the remaining two cases which were not covered in the earlier proof.
	
	\textbf{Case-2: both $\varphi$ and $\psi$ are unsatisfiable.}
	In this case, the only accepted arguments in $F$ under $\sem$ are $B\cup\{b_\psi\}$, resulting in a total of $M+1$ many arguments. 
	Therefore, there are no $K$-diverse $\sem$-extensions in $F$.
	Here, we disallow any smaller admissible sets (e.g., subsets of $B$), even though the maximum achievable diversity in $F$ is still $M$, which remains below $K$.
	
	\textbf{Case-3: $\psi$ is satisfiable but $\varphi$ is unsatisfiable.}
	Here, the sub-AF $F'_\psi$ admits a $\sem$-extension due to the claim in the standard translation.
	In fact, any satisfying assignment $\theta$ for $\psi$ results in a $\sem$-extension for $F$ given as $E_\psi= \theta \cup \{\psi\}$.
	Apart from extensions corresponding to satisfying assignments of $\psi$, the other only other extension in $F'_\psi$ is $\{b_\psi\}$.
	Whereas, in $F'_\varphi$ only arguments in $B$ are contained in any $\sem$-extension.
	In fact, $B$ is the only $\sem$-extension in $F'_\varphi$ since $\varphi$ is not satisfiable as no subset of $B$ can be a $\sem$-extension in $F'_\psi$ due to our auxiliary arguments and attacks from set $C$ of arguments.
	
	Thus, the only extensions in $F$ are $S= B\cup E_1$ and $T= B\cup E_2$ for any two $\sem$-extensions $E_1$ and $E_2$ in $F'_\psi$.
	As a results, $d(S,T)=|E_1|+|E_2|$ which can be at most $2k_2$ since each $E_i$ is either of the form $\theta\cup\{\psi\}$ for a satisfying assignment $\theta$ of $\psi$, or of the form $\{b_\psi\}$.
	Thus, $F$ does not admit two $K$-diverse $\sem$-extensions by our assumption,  
	as $N>k_2$.
	
	This completes our claim as in each case, either $F$ does not admit $K$-diverse $\sem$-extensions, or it admits $K'$-diverse $\sem$-extensions for
	some $K'>K$.
\end{proof}
\end{proof}

\paragraph{For 2nd level semantics}

Before proving our results for semi-stable and stage semantics, we need to slightly alter the standard translation for these two semantics.
Interestingly, the reduction still remains correct, if we add self-attacks for the \emph{clause arguments} in the standard translation.
Let $\Phi = \forall Y \exists Z.\varphi$ be a QBF  where $\varphi\dfn\bigwedge_{i=1}^m C_i$ is a CNF-formula with clauses $C_i$ over variables $X = Y \cup Z$.
As before, we let $Y_*=\{y_*\mid y\in Y\}$ and $\bar Y_*=\{\bar y_*\mid y\in Y\}$ as auxiliary sets of arguments and consider the AF $G_\Phi=(A,R)$:
\begin{align*}
A \dfn & \{\varphi,\bar \varphi, b\} \cup \{C_1 , \dots, C_m\} \cup X \cup \bar X \cup Y_* \cup \bar Y_*\\
R \dfn &\{\,(C_i,\varphi)\mid i\leq m\,\} \cup \{\,(x,\bar x),(\bar x,x)\mid x\in X\,\}\\
&\cup\{\,(x,C_i)\mid x\in C_i\,\}\cup\{\,(\bar x,C_i)\mid \bar x\in C_i\,\} \\
& \cup \{(\varphi, \bar\varphi), (\bar\varphi, \varphi), (\varphi, b), (b, b)\,\}\\
& \cup \{(y,y_*), (\bar y, \bar y_*) ,(y_*,y_*), (\bar y_*,\bar y_*)\mid y\in Y\}\\
& \cup \{(C_i,C_i) \mid i\leq m\}.
\end{align*}
The self-attacks (in the last line) are needed in our claims below to prove certain size assumptions for extensions containing the argument $\bar\varphi$, when $\Phi$ is false. Intuitively, we are forcing that any such extension can not contain clause arguments.
Henceforth, we refer to $G_\Phi$ as the AF obtained from $\Phi$ via the {standard translation for semi-stable semantics, with self-conflicting clause arguments.}

\begin{lemma}\label{lem:self-attcks}
$\Phi$ is false iff $G_\Phi$ admits a semi-stable (stage) extension containing $\varphi$.
\end{lemma}

\begin{proof}
We use the claims for standard translation $F_\Phi$ as an  intermediate steps as $G_\Phi$ only adds some additional attacks.
Suppose $\Phi$ is true.
Then, the claim follows due to the standard translation for both semantics as no $\sem$-extension ($\sem\in\{\semi,\stag\}$) contains the argument $C_i$ for any $i\leq m$.

Next, suppose $\Phi$ is false.
We prove that there exists an $\sem$-extension $E$ in $G_\Phi$ such that $\bar\varphi\in E$.

Let $\theta_Y$ be an assignment for which no assignment $\theta_Z$ exists, with $\theta\models \varphi$ for $\theta=\theta_Y\cup \theta_Z$.
We choose $\theta_Z$ in such a way that $\theta$ satisfies maximum number of clauses of $\varphi$.
Now, we prove that $S= \theta\cup\{\bar \varphi\}$ is a valid $\sem$-extension in $G_\Phi$ for $\sem\in\{\semi,\stag\}$.
It is easy to see that $S$ is conflict-free (resp., admissible, for semi-stable semantics) as $\theta$ is a valid assignment and every literal argument defends itself against its negated literal. Moreover, $\bar\varphi$ attacks its only attacker $\varphi$.
Therefore, we only need to prove the claim regarding the range of arguments in $S$.
Suppose there exists a conflict-free (resp., admissible) set $T$ in $G_\Phi$ with $T^+_R \supsetneq S^+_R$.
This implies that there is some argument $x\in A$ with $x\in T^+_R\setminus S^+_R$.
We consider two cases, depending on whether $\varphi\in T$ or $\bar\varphi\in T$.

\textbf{Case (I): $\varphi\in T$.}
Since, $\theta$ is an assignment, for each a $y\in Y$, either $y\in \theta$ or $\bar y\in \theta$.
Moreover, since $\varphi\in T$, the literals in $T$ yield a satisfying assignment $\beta$ for $\varphi$.
However, by our claim, $\theta \not\models \varphi$ and in particular, for $\theta_Y$, there is no assignment $\theta_Z$ such that $\theta_Y\cup \Theta_Z\models \varphi$.
This implies that there is at least one literal $\ell$ over variable $y\in Y$, such that $\ell\in S$ and $\bar \ell\in T$.
But then, for the corresponding arguments $\ell_*$ and $\bar\ell_*$, we have that $\ell_*\in S^+_R$ and $\bar\ell_*\in T^+_R$.
However, this implies that $\ell_*\not\in T^+_R$ and hence $T^+_T\not\supset S^+_R$.
This leads to a contradiction.

\textbf{Case (II): $\bar\varphi\in T$}.
As we have seen in the first case. $S$ and $T$ must agree on their arguments corresponding to literals over $Y$ for $T^+_R\supset S^+_R$ to be true.
But this leaves us with only one possibility that there is some clause argument $C_i\in T^+_R\setminus S^+_R$.
Due to our assumption, $\theta_Y\subseteq T$.
Notice that $C\in S^+_R$ for the clause argument $C$ iff $C$ is attacked by some argument in $S$, or equivalently, the clause $C\in\varphi$ is satisfied by the assignment $\theta$.
However, we have chosen $\theta_Z$ in such a way that it satisfies maximum number of clauses of $\varphi$.
Therefore, for every clause $C$, we have that either $C\in S^+_R$ or $C\not\in T^+_R$.
This contradicts our assumption and hence $S$ is a stage (semi-stable) extension.
\end{proof}

\thmmaxkdiv*
\begin{proof}[Proof of Statement 2]
For membership, observe that an AF $F=(A,R)$ admits two maximally $k$-diverse $\sem$-extensions iff 
(i) $F$ admits a pair $S_p, T_p$ of $\sem$-extensions with $d(S_p,T_p)=k$, and 
(ii) for each pair $S_n,T_n$ of $\sem$-extensions in $F$, $d(S_n,T_n)\leq k$.
Here, (i) is in $\SigmaP$ and (ii) is in $\PiP$. Thus the problem $\maxkdiverseext_\sem$ is in $\DPtwo$ for $\sem\in \{\semi,\stag\}$.

For hardness, we reduce from the \textbf{SAT-UNSAT problem for $\TQSAT$} by using the standard translation as an intermediate step.
Let $\Phi = \forall Y \exists Z.\varphi$ where $\varphi\dfn\bigwedge_{i=1}^m C_i$ is a CNF over variables $V_\Phi = Y \cup Z$, and $\Psi = \forall Y' \exists Z'.\psi$ where $\psi\dfn\bigwedge_{i=1}^n D_i$ is a CNF over $V_\Psi = Y' \cup Z'$.
{We set $k_1= |V_\varphi|$, $k_2= |V_\psi|$}, and  construct $F_\Phi$ and $F_\Psi$ via the standard translation for semi-stable semantics with self-attacking clause arguments in both sub-AFs.
Then, we obtain (as their disjoint union) an AF $F'=(A',R')$ constructed as follows:
\begin{align*}
	A' \dfn & \{\varphi,\bar \varphi\} \cup \{C_1 , \dots, C_m\} \cup V_\Phi \cup \bar V_\Phi \cup Y_* \cup \bar Y_*\\
	& \cup \{\psi,\bar \psi\} \cup \{D_1 , \dots, D_n\} \cup V_\Psi \cup \bar V_\Psi \cup Y'_* \cup \bar Y'_*, \\
	R' \dfn &\{\,(C_i,\varphi)\mid i\leq m\,\} \cup \{\,(x,\bar x),(\bar x,x)\mid x\in X\,\}\\
	&\cup\{\,(x,C_i)\mid x\in C_i\,\}\cup\{\,(\bar x,C_i)\mid \bar x\in C_i\,\} \\
	& \cup \{(\varphi, \bar\varphi), (\bar\varphi, \varphi), (\varphi, b), (b, b)\,\}\\
	& \cup \{(y,y_*), (\bar y, \bar y_*) ,(y_*,y_*), (\bar y_*,\bar y_*)\mid y\in Y\}\\
	& \cup \{(C_i,C_i) \mid i\leq m\}\\
	& \\
	& \cup \{\,(D_j,\psi)\mid j\leq n\,\} \cup \{\,(x,\bar x),(\bar x,x)\mid x\in V_\Psi\,\}\\
	&\cup\{\,(x,D_j)\mid x\in D_j\,\}\cup\{\,(\bar x,D_j)\mid \bar x\in D_j\,\} \\
	& \cup \{(\psi, \bar\psi), (\bar\psi, \psi), (\psi, b), (b, b)\,\}\\
	& \cup \{(y',y'_*), (\bar y', \bar y'_*) ,(y'_*,y'_*), (\bar y'_*,\bar y'_*)\mid y'\in Y'\}\\
	& \cup \{(D_j,D_j) \mid j\leq n\}.
\end{align*}

We will implicitly apply Lemma~\ref{lem:union} and hence the union of two extensions from the two sub-AFs is a valid extension in the combined AF.
Moreover, by the correctness in the standard reductions (Lemma~\ref{lem:self-attcks}), the formula $\Phi$ (resp., $\Psi$) is true iff no extensions in the sub-AF $F_\Phi$ ($F_\Psi$) contains $\bar\varphi$ ($\bar\psi$). 

Next, to distinguish the positive and negative cases of \textbf{SAT-UNSAT}, we consider two fresh sets  $D_1 = \{d_{i} \mid i\leq k_1\times k_2\}$ and $D_2= \{c_j \mid j \leq 3k_2\}$ of auxiliary arguments.
As before, we set our threshold $K$ in such a way that whenever $\Phi$ is false and $\Psi$ is true, then two $K$-diverse $\sem$-extensions can be found, whereas in all other cases, the number $K$ is not the maximum achievable diversity.
As in the proof of Theorem~\ref{thm:div-arg-semi}, this can be modeled via taking a witnessing extension $S_{\bar\varphi}$ for $\bar\varphi$ (when $\Phi$ is false) and then replacing $\bar\varphi$ in $S_{\bar\varphi}$ by the set $D_1$ of arguments to yield $S_D$.
This allows us to achieves a large value of diversity in the sub-AF $F_\Phi$, namely $k_1\times k_2+1$ by taking these two $\sem$-extensions.

Now, ideally, one could find any extension in $F_\Psi$ to yield two witnessing $\sem$-extensions in $F$.
However, it is also possible to take two different sub-extensions of $F_\Psi$ to yield the final $\sem$-extensions in $F$, which could further increase the diversity. 
Consequently, the diversity of extensions in $F$ can no longer be controlled via taking sub-extensions in the individual sub-AFs. 
Precisely, for any two extensions $T_1,T_2$ of $F_\Psi$ and $S_{\bar\varphi}, S_D$ of $F_\Phi$, $F$ admits $E_1\dfn S_{\bar\varphi}\cup T_1$ and $E_2\dfn S_D\cup T_2$ as valid $\sem$-extensions. But we don't know in advance what $d (E_1, E_2)$ could be and thus we can not set a maximum $K$ solely based on whether $\Psi$ is true.
Similarly, although we know that $d(S_{\bar\varphi}, S_D)= k_1\times k_2+1$, this diversity could still increase when $\varphi$ admits a satisfiable assignment even when $\Phi$ is false.
To be precise, one may consider an extension $S_\varphi\dfn \theta\cup\{\varphi\}$ corresponding to a satisfying assignment $\theta$ such that $\theta\models \varphi$ and such that $S_\varphi\cap S_D=\emptyset$.
Then, $d(S_{\varphi}, S_D)=k_1\times k_2 + k_1+1$ as $\varphi\in S_\varphi$ and $|\theta|=k_1$.
That is, the diversity may increase even when $\Phi$ is false but there is at least one satisfying assignment for $\varphi$.

To overcome the aforementioned issue, we need a slight technical trick in our final AF.
Precisely, we enforce two ``most diverse'' $\sem$-extensions in both sub-AFs ($F_\Phi$ and  $F_\Psi$), and then define $K$ based on such extensions. This models the intuition that for any other pair of extensions, the diversity remains below this maximum value.
To achieve this, we add dummy arguments $\{a_\Phi,\bar a_\Phi\}$ and $ \{ b_\Psi, \bar b_\Psi\}$ in such a way that allows us to select two most diverse $\sem$-extensions in $F_\Phi$ and $F_\Psi$, respectively.
%
Furthermore, {we re-set $k_1= |V_\varphi|+1$ and $k_2= |V_\psi|+1$}, so that the size of an extension in the individual sub-AF is $k_1$ and $k_2$, respectively after adding these auxiliary arguments.
As our final AF, we let $F=(A,R)$, where
\begin{align*}
	A = A' & \cup D_1 \cup D_2\cup\{a_\Phi,\bar a_\Phi \}\cup \{ b_\Psi, \bar b_\Psi\},\\
	R = R' & \cup \{(\varphi,d), (d,\varphi), (\bar\varphi,d), (d,\bar\varphi) \mid d\in D_1\}\\
	& \cup \{(\psi,c), (c,\psi), (\bar\psi,c), (c,\bar\psi)\mid c\in D_2\}\\
	& \cup \{(a_\Phi,\bar a_\Phi), (\bar a_\Phi,a_\Phi)\}\cup \{(\bar a_\Phi, C_i)\mid C_i\in\varphi\}\\
	& \cup \{(b_\Psi,\bar b_\Psi), (\bar b_\Psi,b_\Psi)\}\cup \{(\bar b_\Psi, D_j)\mid D_j\in\psi\}.
\end{align*}

For brevity, we 
write $F^d_X$ for the sub-AF of $F$ consisting of $F_X$ for $X\in\{\Phi,\Psi\}$, and additionally contains the dummy arguments $D_1\cup \{a_\Phi,\bar a_\Phi\}$ or $D_2\cup \{b_\Psi,\bar b_\Psi\}$, and attacks involving them, for $X\in\{\Phi,\Psi\}$, respectively.
Furthermore, the auxiliary arguments in $D_1$ and $D_2$ are only connected to sub-AFs $F^d_\Phi$ and $F^d_\Psi$, respectively.
Thus, the two sub-AFs are still disconnected from each other.
As a result, due to Lemma~\ref{lem:union}, $\sem$-extensions in the combined AF $F$ are obtained precisely via $\sem$-extensions in the sub-AFs for $\sem\in\{\semi,\stag\}$.
Recall that we set $k_1$ and $k_2$ in such a way that, any $\sem$-extension in either sub-AF ($F^d_\Phi$ or $F^d_\Psi$, respectively) corresponding to a any (non-)satisfying assignment has size ($k_1+1$ or $k_2+1$) due to one of the dummy argument ($a_\Phi/\bar a_\Phi$ or $b_\Psi/\bar b_\Psi$) and an argument for the formula ( $\varphi/\bar\varphi$ or $\psi/\bar\psi$) depending on whether the assignment satisfies the formula or not.

\textbf{Intuition:}
The attacks in $F$ involving the arguments $\{a_\Phi,\bar a_\Phi\}$ model the intuition that any satisfying assignment $\theta$ of $\varphi$ yields a $\sem$-extension $S_\theta\dfn \theta\cup\{\varphi,a_\Phi\}$ in $F^d_\Phi$.
Whereas, $\{\bar a_\Phi\}$ alone attacks all the ``clause arguments'' of $\varphi$ and yields a $\sem$-extension in $F_\Phi^d$ containing $\varphi$.
In fact, for any assignment $\theta$ and its corresponding $\sem$-extension $S_\theta$ in $F^d_\Phi$, we get another extension $S_{\bar\theta} \dfn \bar\theta\cup\{\varphi,\bar a_\Phi\}$ where $\bar\theta \dfn \{\bar\ell \mid \ell\in\theta\}$ such that, $\bar\ell = \bar x$ if $\ell =x$ and $\bar\ell = x$ if $\ell=\bar x$.
Here, we observe that $S_\theta\cap S_{\bar\theta}= \{\varphi\}$ and hence $d(S_\theta,S_{\bar\theta})=2k_1$, resulting in two maximally diverse $\sem$-extensions in $F^d_\Phi$ containing $\varphi$.
Likewise, when $\Phi$ is false, we obtain a $\sem$-extension $S_{\alpha}= \alpha \cup\{\bar\varphi, a_\Phi \} $ in $F^d_\Phi$ for a non-satisfying assignment $\alpha$ of $\varphi$.
However, now one can additionally replace $\bar\varphi$ by arguments in $D_1$ to obtain a valid $\sem$-extension in the sub-AF $F_\Phi^d$, following Lemma~\ref{lem:range}. 
We consider the $\sem$-extensions $S_D = (S_\alpha\setminus \{\bar\varphi\}) \cup D_1$ and $S_{\bar\alpha} = \bar\alpha \cup\{\varphi, \bar a_\Phi\}$.
Consequently, {$F^d_\Phi$ admits two $k_1\times k_2 +2k_1+1$ diverse extensions} as $S_\alpha\cap S_{\bar\alpha}=\emptyset$.
The same reasoning applies to $\sem$extensions in the sub-AF $F^d_\Psi$, thus giving two maximally diverse $\sem$-extensions in $F^d_\Psi$ for any (non-)satisfying assignment $\beta$ for $\psi$.
Precisely, when $\Psi$ is true, then every sub-extension in $F^d_\Psi$ contains $\psi$ and two such extensions can be at most $2k_2$-diverse (when their corresponding assignments disagree on all the literals).
Whereas, when $\Psi$ is false, then following same reasoning as for $\Phi$, there is some non-satisfying assignment $\beta$ of $\psi$.
This yields an extension $T=\beta\cup\{\bar\psi, b_\Psi\}$ in $F^d_\Psi$, as well as 
an extension $T_D\dfn (T_\beta\setminus \bar\psi)\cup D_2$, due to Lemma~\ref{lem:range}.
Similarly, we we consider the extension $T_{\bar\beta}\dfn  \bar\beta \cup\{\psi, \bar b_\Psi\}$ and observe that  $T_{\bar\beta}$ is a $\sem$-extension in $F^d_\Psi$ as the argument $\bar b_\Psi$ defends $\psi$ against all attacks from ``clauses arguments''.
Here, we have $|T_D|=4k_2$ (due to $|D_2|=3k_2$ and $|\beta|=k_2$), and $|T_{\bar\beta}|=k_2+1$.
Thus, we obtain a maximum diversity of $5k_2+1$ in this case.

\textbf{To complete our reduction}, we set $K=k_1\times k_2 +2(k_1+k_2)+1$.
Here $k_1\times k_2 +2k_1+1$ accounts for the diversity arising form $\sem$-extensions of $F_\Phi^d$ containing the set $D_1$ and the argument $\bar\varphi$ (when $\Phi$ is false), whereas $2k_2$ accounts for two most diverse $\sem$-extensions of $F^d_\Psi$ due to any two satisfying assignments for $\Psi$. Finally, if $\Psi$ is also false, the diversity goes even higher since the set $D_2$ (of size $3k_2$) is then involved.

For correctness, we prove the following claim.
\begin{claim}\label{claim:semi-DP2}
	The $\TQSAT$ pair $(\Psi,\Phi)$ is a positive instance of SAT-UNSAT iff $F$ has two $K$-diverse $\sem$-extensions but no $K'$-diverse $\sem$-extensions for any $K'>K$.
\end{claim}
\begin{proof}[Proof of Claim]
	\textbf{Suppose $\Psi$ is true and $\Phi$ is false.}
	Then, by Lemma~\ref{lem:self-attcks}, $F_\Phi$ admits a $\sem$-extension $S'_\theta= \theta \cup \{\bar\varphi\}$, and hence $F^d_\Phi$ admits a $\sem$-extension $S_\theta = S'_\theta\cup \{a_\Phi\}$.
	By Claim~\ref{claim:semi-sigmap}, the set $S_D\dfn (S_\theta\setminus \{\bar\varphi\}) \cup D_1$ is a valid $\sem$-extension for $F_\Phi$.
	As before, we define $\bar\theta \dfn \{\bar\ell \mid \ell\in\theta\}$ such that, $\bar\ell = \bar x$ if $\ell =x$ and $\bar\ell = x$ if $\ell=\bar x$.
	Next, we let $S_{\bar\theta} \dfn \bar\theta\cup \{ \bar a_\Phi,\varphi\}$ and observe that $S_{\bar\theta}$ is a $\sem$-extension in $F^d_\Phi$. 
	This holds since the argument $\bar a_\Phi$ alone attacks all the ``clause arguments'' in $F^d_\Phi$ and the remaining literals are taken in such a way to attain the maximum range.
	Moreover, since $\Psi$ is true, any $\sem$-extension $T'_\beta$ for $F_\Psi$ contains $\psi$ and can be written as $T'_\beta\dfn \beta\cup\{\psi\}$ for some assignment $\beta$ such that $\beta\models \psi$.
	We let $T_\beta = T'_\beta \cup \{b_\Psi\}$ and observe that $T_\beta$ is a valid $\sem$-extension in $F^d_\Psi$.
	Likewise, $T_{\bar\beta}\dfn \bar\beta\cup \{\bar b_\Psi\}$ is also a $\sem$-extension in $F^d_\Psi$ following the same reasoning as for $F^d_\Phi$.
	
	To complete the claim proof for this direction, we consider $\sem$-extensions
	$E_1= S_D\cup T_\beta$ and $E_2= S_{\bar\theta}\cup  T_{\bar \beta}$ in $F$. Observe that:
	(i) $|S_D|=k_1\times k_2+k_1$,
	(ii) $|T_\beta|= k_2+1= |T_{\bar\beta}|$, and
	(iii) $|S_{\bar\theta}|= k_1+1$.
	Furthermore, we have
	(iv) $T_{\beta} \cap T_{\bar\beta} = \{\psi\}$, and therefore
	(v) $d(T_{\beta}, T_{\bar\beta})= 2k_2$,
	(v) no $S_i$ shares any argument with $T_j$ for $i\in\{D, \bar\theta\}$ and $j\in\{\beta,\bar \beta\}$.
	As a consequence, we have $d(E_1,E_2)= (k_1\times k_2 + k_1) +(k_2) + (k_1+1) + (k_2)= K$.
	This proves our claim that \textbf{$F$ admits two $K$-diverse $\sem$-extension} for $\sem\in \{\semi,\stag\}$.
	
	We next prove that $K$ is maximal and hence for any extensions $S_1,S_2$ of $F^d_\Phi$ and $T_1,T_2$ of $F^d_\Psi$, $d(E_1,E_2)\leq K$ for $E_i = S_i\cup T_i$ and $i=1,2$.
	
	First, observe that $d(T_1,T_2)\leq 2k_2$ as no $\sem$-extension in $F^d_\Psi$ contains $\bar\psi$ and hence also no argument in $D_2$.
	Then, we consider a case distinction on extensions $S_i$ in $F^d_\Phi$.
	
	\textbf{Case-1} 
	Both $S_i$ correspond to some satisfying assignments of $\varphi$. Then $S_1\cap S_2 = \{\varphi\}$ and we have that $d(S_1,S_2)\leq 2k_1$ as no $S_i$ can contain either the argument $\bar \varphi$ or any argument in $D_1$.
	Then, due to $d(T_1,T_2)\leq 2k_2$, we have that $d(E_1,E_2) \leq 2k_1 + 2k_2$ where $E_1 \dfn S_1\cup T_1$ and $ E_2\dfn S_2\cup T_2$.

	\textbf{Case-2} 
	At least one of $S_i$ (say $S_1$) corresponds to a non-satisfying assignment and thus takes the form $S_1\dfn \theta\cup\{\bar\varphi,a_\Phi\}$ due to Lemma~\ref{lem:self-attcks}.
	First, we construct $S_D\dfn (S_1\setminus \{\varphi\}) \cup D_1$ from $S_1$ and observe that $|S_1|=k_1\times k_2+ k_1$.
	Then, we consider the following two sub-cases to see at most how much diversity can be achieved due to $S_D$ and $S_2$ in the sub-AF $F^d_\Phi$.
	\textbf{(I)} If $S_2\dfn \alpha \cup\{\bar\varphi, a_\Phi\}$, then replacing $\bar\varphi$ by $D_1$ in $S_2$ would decrease the diversity of the pair as $D_1\subseteq S_D\cap S_2$. Thus, the maximum diversity is in fact obtained by the set $S_2$ it self, which results in $d(S_D, S_2) = (k_1\times k_2 + k_1)+ (k_1 + 1)$.
	Finally, \textbf{(II)} if $S_2\dfn \alpha \cup\{\varphi, x\}$ for $x\in\{a_\Phi, \bar a_\Phi\}$. Then, we still have either $d(S_D,S_2)= |S_D|+|S_2|$ if $x= \bar a_\Phi$ or $d(S_D,S_2)= |S_D|+|S_2|-2$ if if $x=  a_\Phi$.
	Together with $d(T_1,T_2)\leq 2k_2$, we obtain that $d(E_1,E_2) \leq (k_1\times k_2 + k_1)+ (k_1 + 1) + 2k_2$ where $E_1 \dfn S_D\cup T_1$ and $ E_2\dfn S_2\cup T_2$.
	
	Hence, in any event, we have $d(E_1,E_2) \leq K$ where $E_1 \dfn S_1\cup T_1$ and $ E_2\dfn S_2\cup T_2$ for any pair of $\sem$-extensions $S_1,S_2$ in $F_\Phi^d$ and $T_1,T_2$ in $F_\Psi^d$.
	This completes the proof to our claim in this direction.
	
	{Conversely, suppose $(\Psi,\Phi)$ is a negative instance of SAT-UNSAT.}
	Then, we have one of the following three cases.
	
	\textbf{Case-1: both $\Phi$ and $\Psi$ are true.}
	Then, each $\sem$-extension $S'$ and $T'$ of the sub-AF $F_\Phi$ and $F_\Psi$ contains $\varphi$ and $\psi$, respectively.
	This yields $\sem$-extensions $S\dfn S'\cup\{a_\Phi\}$ and $T\dfn T' \cup\{b_\Psi\}$ in $F^d_\Phi$ and $F^d_\Psi$, of sizes $k_1+1$ and $k_2+2$, respectively. 
	Hence one can achieve a maximum diversity of $2(k_1+k_2)$ at most, by considering two completely different $\sem$-extensions in both sub-AFs.
	Consequently, $F$ does not have two $K$-diverse $\sem$-extensions.
	
	\textbf{Case-2: both $\Phi$ and $\Psi$ are false.}
	Here, (sub-)AFs $F^d_\Phi$ and $F^d_\Psi$ yield extensions $S_{\theta} \dfn \theta\cup \{\bar\varphi, a_\Phi\}$ and $T_{\beta}\dfn \beta\cup \{\bar\psi, b_\Psi\}$, corresponding to failing assignments $\theta$ and $\beta$ of $\varphi$ and $\psi$, respectively.
	Then, we let $S_D\dfn (S_\theta\setminus\{\bar\varphi\})\cup D_1$ and $T_D\dfn (T_\beta\setminus\{\bar\psi\})\cup D_2$.
	As our first extension in $F$, we let $E_1 = S_D\cup T_D$.
	Moreover, we consider $S_{\bar\theta}\dfn \bar\theta \cup \{\bar a_\Phi, \varphi\}$ and $T_{\bar\beta}\dfn \bar\beta \cup \{\bar b_\Psi, \psi\}$ as two additional $\sem$-extensions in $F^d_\Phi$ and $F^d_\Psi$, respectively.
	As our second witness, we let $E_2= S_{\bar\theta}\cup T_{\bar\beta}$.
	Clearly, $E_1\cap E_2 = \emptyset$ and hence $d(E_1,E_2)= |E_1|+|E_2|$.
	We conclude by observing that $|E_1|= |S_D|+|T_D|= (k_1\times k_2 +k_1)+ (3k_2+k_2)$ and $|E_2|= |S_{\bar\theta}| + |T_{\bar\beta}|= (k_1+1)+(k_2+1)$ and hence $d(E_1,E_2)> K$.
	As a result, we obtain two $\sem$-extensions in $F$ which are $K'$-diverse where $K'>K$.
	This is a contradiction to $K$ being the maximal achievable diversity.
	
	\textbf{Case-3: $\Phi$ is true and $\Psi$ is false.}
	In this case, any $\sem$-extension $S$ for $F_\Phi$ contains $\varphi$ and hence no $\sem$-extension contains either $\bar\varphi$ or any argument in $D_1$.
	Therefore, any extension of $F^d_\Phi$ has size $k_1+1$ at most.
	Two most diverse extensions in $F^d_\Phi$ are thus $S_\theta \dfn \theta\cup \{a_\Phi, \varphi\}$ and $S_{\bar\theta} \dfn \bar\theta\cup  \{\bar a_\Phi, \varphi\}$ with $d(S_\theta, S_{\bar\theta})=2k_1$.
	
	Then, we construct two most diverse $\sem$-extensions in$F^d_\Psi$.
	Since $\Psi$ is false, there is an assignment $\beta$ such that $T_\beta \dfn \beta \cup \{b_\Psi, \bar\psi\}$ is a $\sem$-extension in $F^d_\Psi$.
	Now, we let $T_D\dfn (T_\beta\setminus\{\bar\psi\})\cup D_2$.
	Moreover, we let $T_{\bar\beta}\dfn \bar\beta\cup\{\bar b_\Psi, \psi\}$ as our second $\sem$-extension (as $\bar b_\Psi$ defends $\psi$) in $F^d_\Psi$.
	It can be observed that $T_D$ and $T_{\bar\beta}$ are two most diverse $\sem$-extensions in $F^d_\Psi$ and that $d(T_D,T_{\bar\beta})=3k_2+(k_2+1)$.
	Now, together with two most diverse $\sem$-extensions $S_\theta,S_{\bar\theta}$ in $F^d_\Phi$, we obtain $E_1= S_\theta \cup T_D$ and $E_2= S_{\bar\theta}\cup T_{\bar\beta}$ as two most diverse $\sem$-extensions in $F$.
	Nevertheless, we have $d(E_1, E_2)=2k_1+4k_2+1$ and one cannot achieve larger diversity values in $F$.
	Hence, once again $F$ does not have two $K$ diverse $\sem$-extensions.
	
	In summary, the only way to achieve a diversity of $K= k_1\times k_2+2(k_1+k_2)+1$ is when $\Phi$ is false and $\Psi$ is true. The table below highlights all the subcases in the proof of our claim.\\
	\begin{tabular}{ccc}
		$\Phi$ & $\Psi$  & max diversity \\
		true & true & $2k_1+2k_2$ \\
		{true} & {false} & $2k_1+4k_2+1$ \\
		false & true &  $k_1\times k_2 + 2(k_1+k_2)+1$ \\
		false & false & $k_1\times k_2 + 2(k_1+k_2)+{3k_2+2}$ \\
	\end{tabular}\\
	
\end{proof}

\end{proof}

We next establish that the argument version of the diversity problem is also $\DPtwo$-complete for $\sem\in\{\semi,\stag\}$.
\begin{restatable}[$\star$]{theorem}{maxext}\label{thm:max-k-div-semi-arg}
The problem $\maxkdiversearg_\sem$ is $\DPtwo$-complete for each $\sem\in\{\semi,\stag\}$.
\end{restatable}
\begin{proof}
The proof follows due to the construction in the proof Theorem~\ref{thm:max-k-div}.
Intuitively, we need to select two \emph{signal} arguments which are maximally $K$-diverse  in $F$ iff $\Psi$ is true and $\Phi$ is false.
Interestingly, we can use our auxiliary arguments $a_\Phi,\bar a_\Phi$ in $F^d_\Phi$ to serve this purpose.
Then, we assert the following claim, whose proof follows the exact same reasoning as the proof of Claim~\ref{claim:semi-DP2}.
\begin{claim}\label{claim:semi-DP2-arg}
	The $\TQSAT$ pair $(\Psi,\Phi)$ is a positive instance of SAT-UNSAT iff the argument $a_\Phi$ and $\bar a_\Phi$ are maximally $K$-diverse in $F$  under semantics $\sem\in\{\semi,\stag\}$.
\end{claim}
This holds since the witness pair of $\sem$-extensions contain the arguments $a_\Phi$ and $\bar a_\Phi$, in each case of the proof for Claim~\ref{claim:semi-DP2} .

\end{proof}

\clearpage
\section{Further Details on Experiments}
\begin{figure}[!h]
\centering \includegraphics[width=.75\textwidth]{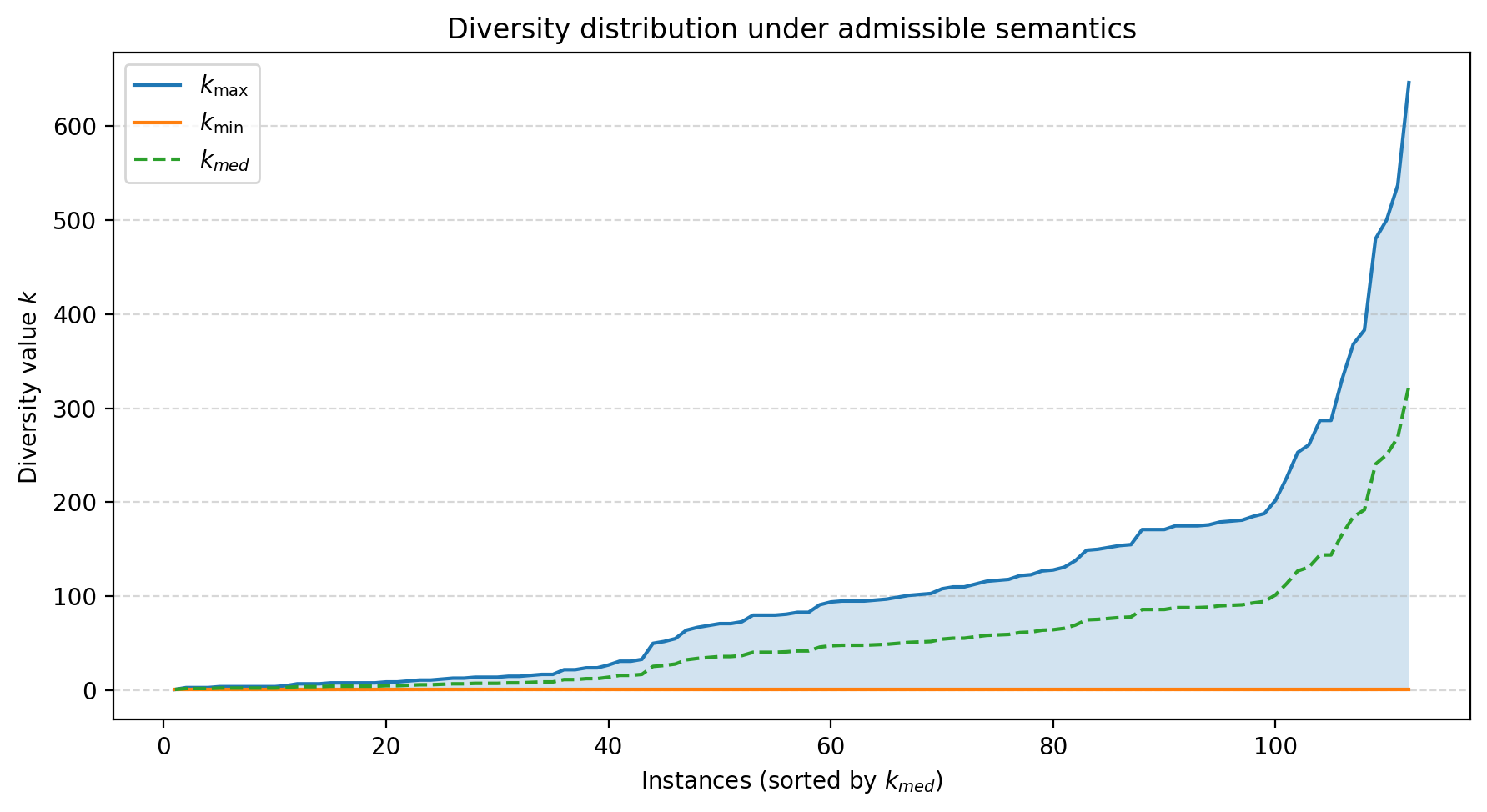}
\caption{Distribution of the minimum, median, and maximum diversity
	for the instances under admissible semantics. The instances are sorted
	by median diversity.}
\label{fig:exfig2}
\end{figure}

\begin{figure}[!h]
\centering \includegraphics[width=.75\textwidth]{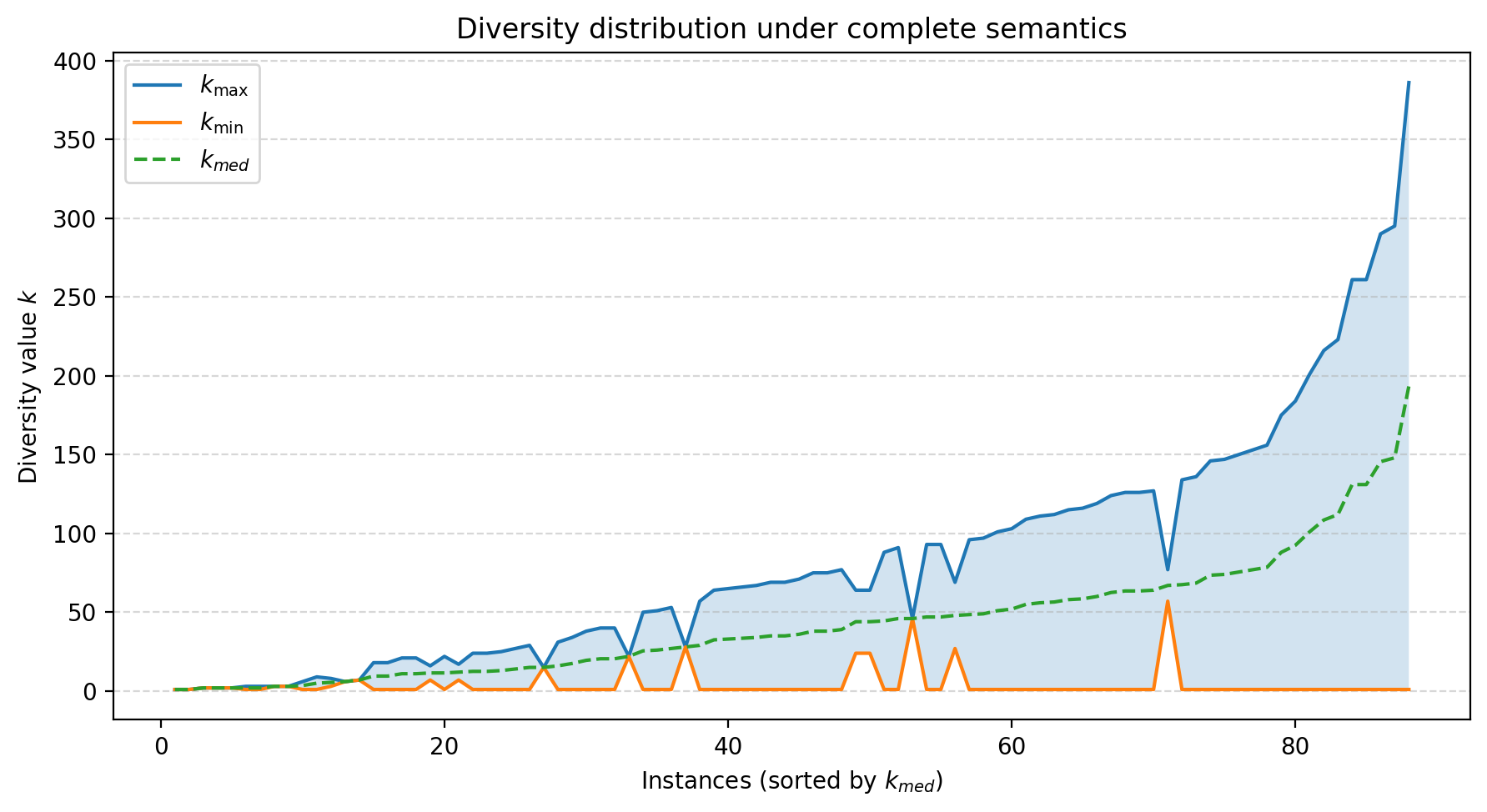}
\caption{Distribution of the minimum, median, and maximum diversity
	for the instances under complete semantics. The instances are sorted
	by median diversity.}
\label{fig:exfig3}
\end{figure}

Figures~\ref{fig:exfig2} and~\ref{fig:exfig3} report the instance-wise distribution of diversity
values under admissible and complete semantics, respectively.
As in the stable case discussed in the main text, diversity values
exhibit a broad range across instances, indicating substantial
variation in the structural differences between extensions.
Under admissible semantics, the minimal diversity remains at one for
almost all instances, reflecting the permissive nature of the
semantics, where nearly identical extensions can coexist.
Nevertheless, the maximal diversity shows a pronounced increase for a
subset of instances, suggesting that even under admissible semantics,
certain frameworks admit highly divergent extensions.
For complete semantics, the observed patterns lie between those of
admissible and stable semantics.
While minimal diversity occasionally attains values larger than one, it
remains relatively small for most instances, whereas maximal diversity
increases gradually and reaches comparatively high values for harder
instances.

Observe that the graphs for stable~(Figure~\ref{fig:exfig}) and complete semantics~(Figure~\ref{fig:exfig3}) look somewhat similar but the one for the case of admissible~(Figure~\ref{fig:exfig2}) appears different.
This can best be explained via the definitions of semantics: admissible extensions only require conflict-freeness and defending their elements, whereas stable and complete also require certain closure properties (stable: attack all arguments not in the set, complete: contain all the arguments defended by the set). 
As a result, some diversity values are only attained by admissible semantics, as it could be that the witnessing admissible extensions are neither complete nor stable. This phenomenon is also illustrated in Example~\ref{ex:levels-adm}.

\end{document}